\documentclass{llncs}
\usepackage{graphicx}
\usepackage{fixltx2e}
\usepackage{mathtools}
\usepackage[nolist,nohyperlinks]{acronym}
\usepackage[section]{placeins}
\usepackage{pdflscape}
\usepackage[scientific-notation=true, round-precision=2, round-mode=figures]{siunitx}

\usepackage{multirow}
\usepackage{rotating}
\usepackage{textcomp}
\usepackage{color, colortbl}

\usepackage{caption}
\usepackage{times}
\usepackage{epsfig}
\usepackage{graphicx}
\usepackage{amsmath}
\usepackage{amssymb}
\usepackage{euscript}
\usepackage{color}
\usepackage{algorithm}
\usepackage{algorithmic}
\usepackage{url}
\usepackage{bm}
\usepackage{multirow}
\usepackage{epstopdf}
\usepackage{soul}
\usepackage{cite}
\usepackage[normalem]{ulem}
\usepackage{bigints}

\usepackage{epstopdf}
\usepackage{subfigure}
\usepackage{enumerate}
\usepackage[utf8]{inputenc}

\usepackage{bm}
\DeclareMathOperator*{\argmin}{arg\,min}

\newcommand{\bx}{\mathbf{x}}
\newcommand{\bs}{\mathbf{s}}

\newcommand{\bt}{\mathbf{t}}
\newcommand{\bu}{\mathbf{u}}
\newcommand{\eps}{\varepsilon}
\newcommand{\btheta}{\boldsymbol{\theta}}
\newcommand{\bphi}{\boldsymbol{\phi}}

\newcommand{\bw}{\mathbf{w}}
\newcommand{\cX}{\mathcal{X}}
\newcommand{\cI}{\mathcal{I}}

\newcommand{\cA}{\mathcal{A}}

\newcommand{\cC}{\mathcal{C}}

\newcommand{\cP}{\mathcal{P}}

\newcommand{\cO}{\mathcal{O}}

\newcommand{\cL}{\mathcal{L}}

\providecommand{\norm}[1]{\lVert#1\rVert}

\DeclareRobustCommand\onedot{.}
\def\@onedot{\ifx\@let@token.\else.\null\fi\xspace}

\def\eg{\emph{e.g}\onedot} 
\def\ie{\emph{i.e}\onedot}

\def\etal{\emph{et al}\onedot}

\definecolor{gray}{rgb}{0.96,0.96,0.96}

\usepackage[acronym]{glossaries}
\makeglossaries

\begin{document}

\title{Maximum Consensus Parameter Estimation by Reweighted $\ell_1$ Methods}

\author{Pulak Purkait\inst{1} \and 
Christopher Zach\inst{1} \and Anders Eriksson\inst{2} 
}
\institute{
Toshiba Research Europe \and  
Queensland University of Technology
}

\maketitle


\begin{abstract}
Robust parameter estimation in computer vision is frequently accomplished by solving the maximum consensus (MaxCon) problem. Widely used randomized methods for MaxCon, however, can only produce {random} approximate solutions, while global methods are too slow to exercise on realistic problem sizes. Here we analyse MaxCon as iterative reweighted algorithms on the data residuals. We propose a smooth surrogate function, the minimization of which leads to an extremely simple iteratively reweighted algorithm for MaxCon. We show that our algorithm is very efficient and in many cases, yields the global solution. This makes it an attractive alternative for randomized methods and global optimizers. The convergence analysis of our method and its fundamental differences from the other iteratively reweighted methods are also presented. 
\end{abstract}


\section{Introduction}\label{sec:intro}

Robust estimation of model parameters is a critical task in computer vision~\cite{meer04}. The literature on robust estimators is vast~\cite{huber2011robust}, encompassing different robust criteria and the associated algorithms. In computer vision, however, maximum consensus (MaxCon) is one of the most widely used robust criteria. Accordingly, algorithms for solving MaxCon have been researched extensively in recent years and people have developed a number of ways to solve this. 
In this article we seek for a fast iterative method for model estimation under MaxCon criterion. 

\begin{definition}{\bf MaxCon criterion}
Given a set of measurements $\cX = \{ \bx_i \}^{n}_{i=1}$, find the model parameters $\btheta \in \mathbb{R}^d$ that agree with as many of the data as possible. \ie, 
\begin{equation} 
\begin{aligned} 
\max_{\btheta,\; \cI \subseteq \cX} & 
~~ |\cI| 
~~ \text{subject to} & 
 r(\btheta; \bx_i) \leq \epsilon, \; \forall \bx_i \in \cI,
\end{aligned} \tag{P1}\label{eq:maxcon}
\end{equation}
where $r(\btheta; \bx_i)$ is the {absolute value} of the residual of $\btheta$ at the point $\bx_i$, and $\epsilon$ is the inlier threshold. The point set $\cI$ is called the consensus set w.r.t. $\btheta$. A data point $\bx_i$ is called an inlier w.r.t. $\btheta$ if $\bx_i \in \cI$; otherwise, it is called an outlier. 
\end{definition}

Problem~\eqref{eq:maxcon} can also be written by introducing slack variables, one for each data point, as follows:
\begin{equation} 
\begin{aligned} 
\min_{\btheta,\; \bs} &
~~~ \sum_{i=1}^n \mathbf{1}(s_i ) &
\text{subject to} & 
~~~ r(\btheta; \bx_i) \leq \epsilon + s_i, & s_i \ge 0, 
\end{aligned} \tag{P2}\label{eq:maxcon3}
\end{equation}
where $ \mathbf{1}(s_i)$ is an indicator function that returns $1$ if $s_i$ is \emph{non-zero}. 
Effectively, a point $\bx_i$ with a strictly positive slack $s_i$ is regarded as an outlier. Formulation~\eqref{eq:maxcon3} thus seeks the MaxCon solution by minimizing the number of outliers. The equivalence between the formulations  \eqref{eq:maxcon} and \eqref{eq:maxcon3} can be easily established. The optimized slack values can be interpreted as \emph{shrinkage residuals}, to borrow a term from the area of shrinkage operators~\cite{beck2009fast}. 
In most of the geometric problems, the residuals are linear or quasiconvex~\cite{olsson2010generalized}. The quasiconvex functions have convex sub-level sets and the constraints $r(\btheta; \bx_i) \leq \epsilon + s_i, s_i \ge 0$ form a convex set $\cC$ under quasiconvex (or linear) residuals. 

Now the question is whether minimizing the piecewise objective of \eqref{eq:maxcon3} under the convex constraints $\cC$ is an easy problem? In the following Lemma we show that the set of stationary points of \eqref{eq:maxcon3} is in-fact  the feasible set $\cC$ itself. This makes the problem difficult to optimize. 
\begin{lemma}\label{lem:one}
Any feasible point $(\btheta^\ast,\bs^\ast) \in \cC$ is a local minimum of \eqref{eq:maxcon3}. 
\end{lemma}
\begin{proof}
 Let $\cO^\ast$ be the support set of $\bs^\ast$, \ie, $\cO^\ast = \{i : s_i^\ast > 0\}$. Let $s^{*+} := \min_{i \in \cO^\ast} s_i^\ast$ and $\delta \in (0,\; s^{*+})$. Then the MaxCon objective $|\cO| = \sum_i 1(s_i > 0)$
is non-decreasing in the max-norm neighbourhood
\begin{equation}
  N_\delta :=  \{ \bs \in \mathbb{R}^n : s_i \ge 0, \norm{\bs - \bs^\ast}_\infty \le \delta \} . \nonumber
\end{equation}
The above is true because, by construction for any feasible $\bs \in N_\delta$ has at least the same number of non-zeros as $\bs^\ast$. 
If $\bs \in N_\delta$ is not feasible, it leads to the infinite objective. Thus, the MaxCon objective $|\cO|$ is not lower than the value at $(\btheta^\ast,\bs^\ast)$ in the neighbourhood $\bs \in N_\delta$. In summary, all feasible points $(\btheta^\ast,\bs^\ast) \in \cC$ are local minima of \eqref{eq:maxcon3} in a neighbourhood of $(\btheta^\ast,\bs^\ast)$. \qed 
\end{proof} 
Thus MaxCon is a combinatorial optimization problem that is very challenging. It is typically approached by randomized sample-and-test methods, primarily \texttt{RANSAC}~\cite{fischler1981random} and its variations~\cite{choi09,torr2000mlesac,raguram2008comparative,raguram2013usac}. These randomized sampling methods are limited to a ``simple model'', \ie, would not work for Bundle Adjustment (or translation registration). Moreover, the random nature of the algorithms results in approximate solutions with no guarantees of local or global optimality; indeed sometimes the result can be far from the optimal.  
Presently, several globally optimal algorithms exist~\cite{olsson08,enqvist12,chincvpr2015,li09,zheng11}, however, they are usually based on branch-and-bound or brute force search, thus, they are only practical for small problem sizes $n$.

What is surely missing, therefore, is an \emph{efficient} and \emph{deterministic} algorithm for the MaxCon problem. A number of variations of \texttt{RANSAC} are available, \eg,  \texttt{LO-RANSAC}~\cite{chum03,lebeda2012fixing}, nonetheless, these methods follow similar mechanism of \texttt{RANSAC}. \texttt{MLEsac}~\cite{torr2000mlesac,tordoff2005guided} optimizes a (slightly) different criterion than MaxCon. Although, both are MLEs -- a noise model with uniform inliers and outliers is utilized in MaxCon; in contrast, \texttt{MLEsac} utilizes Gaussian
inliers and uniform outliers. In this work, we develop an iterative refinement scheme for MaxCon optimization \eqref{eq:maxcon3} that produces near optimal solutions. Thus, the proposed method lies in-between fast but very approximate solutions and superior but slow global optimal solution.




\section{Iterative Reweighted $\ell_1$ methods} \label{sec:background}

The convex relaxation to~\eqref{eq:maxcon3} is the minimization of  absolute sum of the shrinkage residuals (assumed bounded) 
\begin{equation}\label{eq:maxcon4} 
\begin{aligned}
& ~~~~~~~~\min_{\btheta,\; \bs}
& &  \sum_{i=1}^n s_i &
& \text{subject to}
& & r(\btheta; \bx_i) \leq \epsilon + s_i, & s_i \ge 0, \\
\end{aligned}
\end{equation}
which is also a robust estimation of the model parameters $\btheta$. 
 Olsson \etal~\cite{olsson10} {used this formulation} for outlier removal by iteratively solving~\eqref{eq:maxcon4} and removing the points with positive shrinkage residuals. 
Since $\ell_1$ norm is linear, \eqref{eq:maxcon4} optimizes a linear objective functional under convex constraints $\cC$ and hence can be solved efficiently with the existing optimizers~\cite{vanderbei2000loqo, wachter2006implementation}. 

The difference between the objective of \eqref{eq:maxcon3} and \eqref{eq:maxcon4} is in how the weighting of the magnitude of $\bs$ affects the optimal solution. Specifically, the larger coefficients are penalized more heavily in \eqref{eq:maxcon4} than smaller coefficients, unlike in \eqref{eq:maxcon3} where positive magnitudes are penalized equally. 
%

\subsection{Proposed Smooth Surrogate function}
The MaxCon \eqref{eq:maxcon3} cannot be solved directly due to the presence of a large number of local solutions. We utilize the regularized smooth surrogate $G_\gamma (\bs) = \sum_{i=1}^n \log (s_i + \gamma)$ to reduce the number of local solutions of $\ell_0$, and arrive at the following constrained concave minimization problem,
\begin{equation}
\begin{aligned}
& ~~~~~~~~ \min_{\btheta,\; \bs}
& & \sum_{i=1}^n \log(s_i + \gamma) 
& \text{subject to}
& & r(\btheta; \bx_i) \leq \epsilon + s_i, && s_i \ge 0.
\end{aligned} \tag{P3}\label{eq:dvrsty}
\end{equation}
$\gamma$ is a parameter chosen as a small positive number to ensure the measure is bounded from below, since $s_i$ can become vanishingly small. This damping factor $\gamma$ can also be observed as a regularization of the optimization \cite{chartrand2008iteratively}. 


\subsection{Minimization of the smooth surrogate function}\label{sec:analytical} 
The general form of \eqref{eq:dvrsty} under the convex constraints $\cC$ 
\begin{equation}
\begin{aligned}
 \min_{\mathbf{u}} f(\mathbf{u}) &~~~~ \text{   subject to } & \mathbf{u} \in \cC, 
\end{aligned}
\end{equation}
where  $f$ is concave and $\cC$ is convex. As a concave function $f$ lies below its tangent, one can improve upon a guess $\mathbf{u}$ of the solution by minimizing a linearisation of $f$ around $\mathbf{u}$. This yields the following iterative algorithm
\begin{eqnarray}
\begin{aligned}[clc]\label{eq:iter}
 \mathbf{u}^{(l+1)} & :=  \arg \min_{\mathbf{u}\in \cC} f(\mathbf{u}^{(l)}) + \left\langle \nabla f(\mathbf{u}^{(l)}),\; (\mathbf{u} - \mathbf{u}^{(l)}) \right\rangle 
 : =  \arg \min_{\mathbf{u}\in \cC} \left\langle \nabla f(\mathbf{u}^{(l)}),\; \mathbf{u} \right\rangle,  
\end{aligned}
\end{eqnarray}
with the initialization $\mathbf{u}^{0} \in \cC$. Each iteration is now the solution to a convex problem \cite{boyd2004convex}. For~\eqref{eq:dvrsty},  substituting $\nabla G_\gamma (\mathbf{s}) = \left[{1}/{s_i + \gamma}\right]$ in~\eqref{eq:iter} yields
 \begin{equation}\label{eq:iteranalysis} 
\begin{aligned}
& (\btheta^{(l+1)},\; \bs^{(l+1)}) 
:= \argmin_{(\btheta,\; \bs) \in \cC}
&&  \sum_{i=1}^n {s_i}/({s_i^{(l)} + \gamma}). 
\end{aligned}
\end{equation}
Defining $w^{(l)}_i = (s_i^{(l)} + \gamma)^{-1}$, we obtain the {proposed iterative reweighted method}: at each iteration it solves the following weighted problem
\begin{equation}
\left.
\begin{aligned}
 ~~~~~~(\btheta^{(l+1)},\; \bs^{(l+1)}) := \argmin_{\btheta,\; \bs} \sum_{i=1}^n {w_i^{(l)}} s_i \\
 \text{subject to~~~}  r(\btheta; \bx_i) \leq \epsilon + s_i, ~
 s_i \ge 0,~~~~~~\\
 ~~~~~~~~~~~~~w^{(l)}_i  := (s_i^{(l)} + \gamma)^{-1}.~~~~~~~~
\end{aligned}
\right\rbrace \tag{S1} \label{eq:algorithm} 
\end{equation}
The details of the initializations are in Section~\ref{sec:initi}. Note that computation of a step-size or a line-search is not required which can significantly speed-up the computation. 

As the residuals of most of the 3D geometric problems under study are quasiconvex \cite{kahl08}, our algorithm is guaranteed to converge (shown in Section~\ref{sec:convergence}) for any $r(\btheta;\bx_i)$ that is quasiconvex. Thus, the proposed  algorithm minimizes a linear objective under quasiconvex residuals~\cite{olsson2010generalized}. This motivates us to call proposed algorithm {IR-LP} to distinguish it from traditional IRL1. Note that under linear residuals, IR-LP solves only a linear program (LP) in each iteration. 

\paragraph{Other properties of \eqref{eq:dvrsty}} Let $(\btheta^\ast,\bs^\ast)$ be a minimizer of \eqref{eq:dvrsty}. 
Then, the Lagrangian is given by
\begin{equation}
\begin{aligned}
  {\cal L}(\bs, \btheta; \lambda, \mu) &= \sum_i \Big( \log(s_i + \gamma) \Big. 
 + \Big. \lambda_i \big( r(\btheta;\; \bx_i) - \eps - s_i \big) - \mu_i s_i \Big)
  \end{aligned}
\end{equation}
where $\lambda_i \ge 0$ and $\mu_i \ge 0$ are Lagrange multipliers. The KKT conditions are as follows
\begin{equation}\label{equ:IRL1KKT}
\begin{aligned} 
 &  \frac{1}{{s_i^\ast} + \gamma} - \lambda_i -  \mu_i = 0,~ \sum_{i=1}^n \lambda_i \nabla_{\btheta}{r(\btheta^\ast;\;\bx_i)} = 0 &\\
 & \lambda_i[r(\btheta^\ast; \bx_i) - \epsilon - s_i^\ast] = 0,\; \mu_i s_i^\ast = 0&\\
  & r(\btheta^\ast; \bx_i) \le \epsilon + s_i^\ast,\; s_i^\ast \ge 0,\; \lambda_i \ge 0,\; \mu_i \ge 0.&
\end{aligned}
\end{equation}
From the first condition, we know $\lambda_i + \mu_i = 1/(s_i^* + \gamma) >
0$ which implies both of $\mu_i$ and $\lambda_i$ can not be zero simultaneously. Hence, for each $i$, one of the constraints $s_i^* \ge 0$ or
$r(\btheta^*;\bx_i) \le \eps + s_i^*$ is always active. A local minimum $(\btheta^\ast,\bs^\ast)$ is, thus, 
characterized by
\begin{equation}
  s_i^* =
  \begin{cases}
    0 & \text{if } i \in \cI \\
    r(\btheta^*;\bx_i) - \eps & \text{if } i \in \cO,
  \end{cases} \nonumber
\end{equation}
where $\cO$ is the support set of $\bs^*$. $\cO$ can also be considered as an outlier set as $s_i^* > 0$  corresponds to an outlier point. Thus, $\cI$ can be considered as an inlier set. Note that for $i \in \cO$, $\lambda_i = 1/(s_i^* + \gamma)$ and for $i \in \cI$, $\lambda_i = 0$. Thus by \eqref{equ:IRL1KKT},   
\begin{equation}
\sum_{i:r(\btheta^\ast;\;\bx_i) > \epsilon} \frac{ \nabla_{\btheta}{r(\btheta^\ast;\;\bx_i)}}{r(\btheta^\ast;\;\bx_i) - \epsilon + \gamma} = 0, 
\end{equation}
which says that the weighted sum of the gradients corresponding to the outliers at a minimum $(\btheta^\ast,\bs^\ast)$ vanishes.  
However a direct relationship with the optimal choice of $\gamma$ and the number of outliers can not be derived which would have given a potential choice of $\gamma$. The choices of $\gamma$ are further discussed in Section \ref{sec:initi}. 


\begin{figure} \center 
\subfigure[Maxcon \eqref{eq:maxcon3}]{\includegraphics[width=0.48\textwidth]{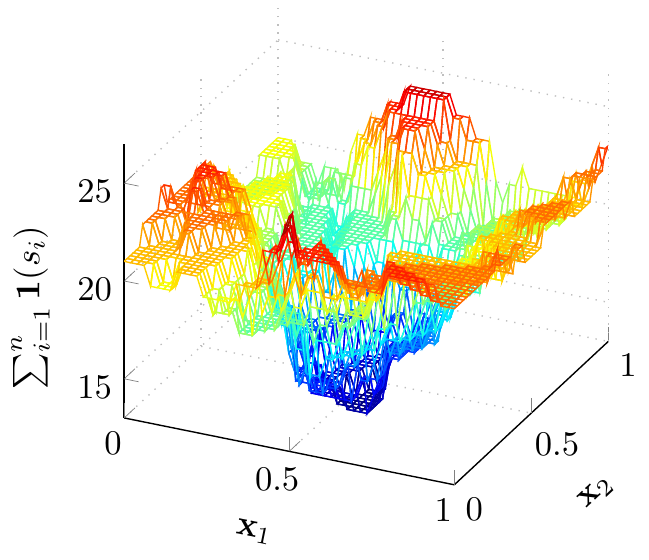}\label{fig:ell1}} 
\subfigure[Shrinkage $\ell_1$ \eqref{eq:maxcon4} ]{\includegraphics[width=0.48\textwidth]{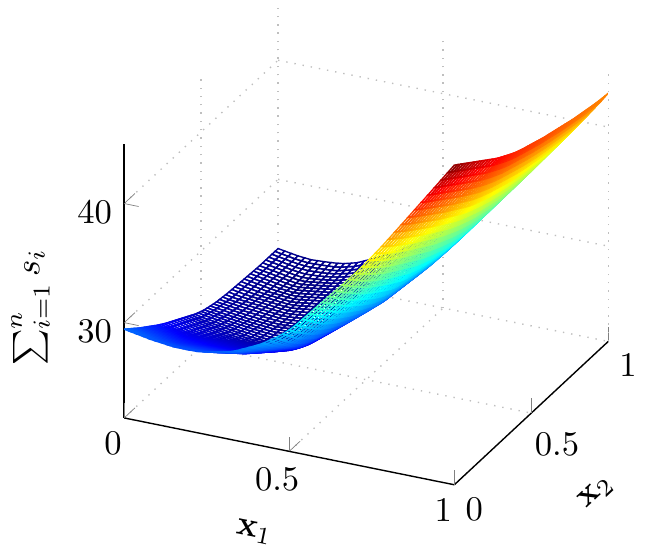}\label{fig:ell0}} \\
\subfigure[$G_\gamma (\bs)$ \eqref{eq:dvrsty}, $\gamma = 0.0001$]{\includegraphics[width=0.48\textwidth]{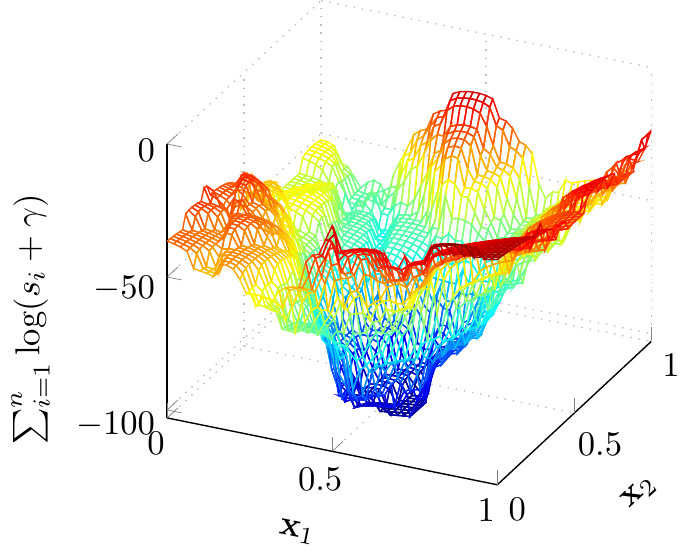}\label{fig:gamma00001}} 
\subfigure[$G_\gamma (\bs)$ \eqref{eq:dvrsty}, $\gamma = 0.001$]{\includegraphics[width=0.48\textwidth]{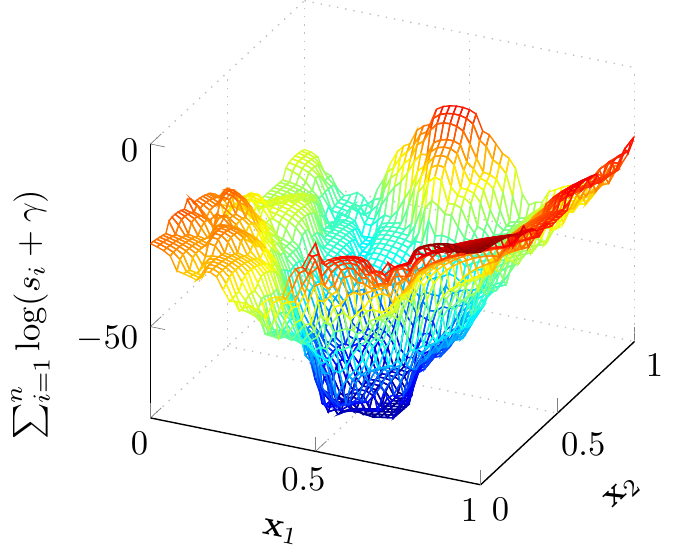}\label{fig:gamma0001}} \\
\subfigure[$G_\gamma (\bs)$ \eqref{eq:dvrsty}, $\gamma = 0.01$]{\includegraphics[width=0.48\textwidth]{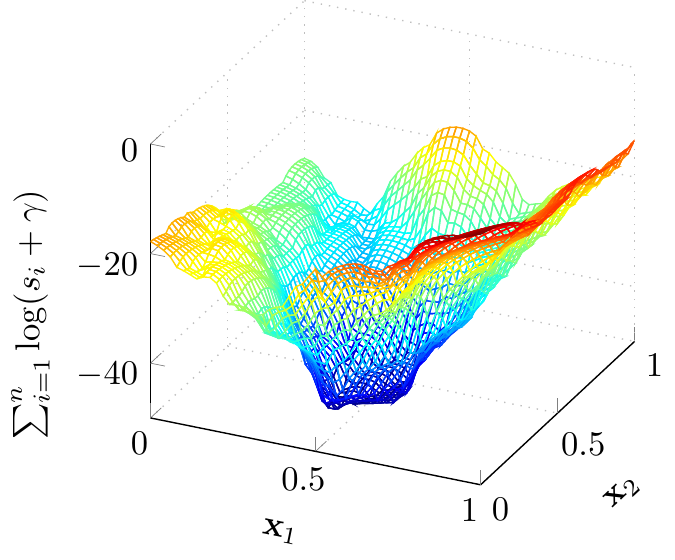}\label{fig:gamma001}} 
\subfigure[$G_\gamma (\bs)$ \eqref{eq:dvrsty}, $\gamma = 0.1$]{\includegraphics[width=0.48\textwidth]{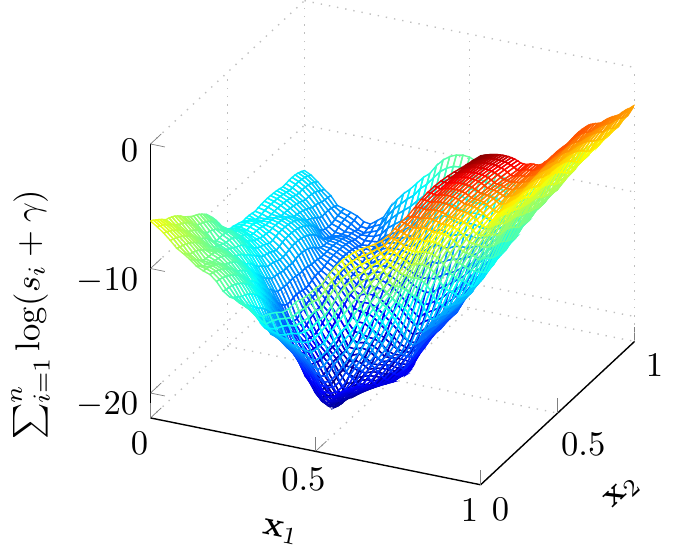}\label{fig:gamma01}} 
\caption{Different objectives are plotted on a synthetic data.}
\label{fig:logComparison}
\end{figure}

Compared to \eqref{eq:maxcon3}, where all feasible points are local minima due to lemma \ref{lem:one}, \eqref{eq:dvrsty} reduces the number of local minima by increasing $\gamma$. 
In Figure \ref{fig:logComparison}, we display the objective of \eqref{eq:dvrsty} on a synthetic 2D line fitting problem, under different values of $\gamma$. As $\gamma$ increases, the topographic surface of the objective function is flatten and fewer local minima are observed. This is an empirical evidence that $G_\gamma$ smoothens the objective of $\eqref{eq:maxcon3}$ in a sensible way. The choice of $\gamma$ is discussed further in Section \ref{sec:parameters}. 

\subsection*{The connection with basis pursuit} 
In the basis pursuit problem, one aims to recover the sparsest signal $\btheta \in \mathbb{R}^d$ from the measurements $\mathbf{y} \in \mathbb{R}^n$, with respect to a dictionary $\bphi \in \mathbb{R}^{n \times d}$:
\begin{equation} \label{eq:basisprst}
\begin{aligned}
\min_{\btheta} & 
~~~\sum_{k=1}^d \mathbf{1}(\theta_k ) 
~~\text{subject to} & 
\mathbf{y} = \bphi \btheta.
 \end{aligned} 
\end{equation}
{Candes \etal~\cite{candes2008enhancing} also proposed a smooth surrogate $\sum_{k=1}^d\log(|\theta_k|+\gamma)$ of the objective above that results an iteratively reweighted $\ell_1$-norm minimization (IRL1) algorithm for~\eqref{eq:basisprst}. Specifically, at the $l$-th iteration, the following weighted $\ell_1$ problem is solved }
\begin{equation} \label{eq:basisirl1}
\begin{aligned}
\btheta^{(l+1)}  := & \argmin_{\btheta}
~~ \sum_{k=1}^d w^{(l)}_k | \theta_k | 
& \text{subject to } 
 ~~~\mathbf{y} = \bphi \btheta,\\
w^{(l)}_k  := & (|\theta_k^{(l)}| + \gamma)^{-1}.
\end{aligned}
\end{equation}
Though related,~\eqref{eq:maxcon3} and~\eqref{eq:basisprst} are quite different  problems. 
\begin{itemize}
\setlength\itemsep{0.0em}\setlength\parsep{-0.0em}
\item The former seeks sparsity on the shrinkage residuals $\bs$ (parameters $\btheta$ allowed to be dense), while the latter seeks sparsity in $\btheta$.
\item Further, the constraints in~\eqref{eq:maxcon3} are usually over-determined ($n > d$), while for~\eqref{eq:basisprst} the constraints are  under-determined ($d > n$). 
\item Moreover, the proposed method~\eqref{eq:algorithm} can also be treated as maximization of residual diversity~\cite{gorodnitsky1997sparse,chartrand2007exact}. Interested readers are referred to the extended version. 
\end{itemize}
Although, the proposed reweighted algorithm is inspired by Candes \etal~\cite{candes2008enhancing}, above set our work apart from~\cite{candes2008enhancing} that has different theoretical underpinnings. Thus, the methods for basis pursuit problems cannot be directly adapted here.

\section{Convergence analysis}\label{sec:convergence}
In this section, we analyse the convergence of the proposed  algorithm \eqref{eq:algorithm}. Let $\cA:U\to\cP(U)$ be an algorithm defined on a set $U$ where $\cP(U)$ is the power set of $U$. 
Given $\cA$, Zangwill's global convergence theorem \cite{sriperumbudur2012proof} is stated as 
\begin{theorem}\label{th:gconverge}
Let $\cA:U\to\cP(U)$ generate a sequence $\{\bu^{(l)}\}_{l=0}^\infty$ through the iteration $\bu^{(l+1)} \in \cA(\bu^{(l)})$, given an initialization $\bu^{(0)} \in U$. Let $\Gamma \subset U$ be a set called solution set. 
Further, let $\cA$ satisfy the following constraints 
\begin{enumerate}[C1.] 
\setlength\itemsep{0.0em}
\item The points in $\{\bu^{(l)}\}$ are contained in a compact subset. 
\item If $\Gamma$ is the solution space of $\cA$, then, there is a continuous function $\cL(\bu):U \to \mathbb{R}$ satisfying
\begin{equation}
\begin{cases} 
\cL(\bu^{(l+1)}) < \cL(\bu^{(l)})  & \text{if } \bu^{(l)} \not\in \Gamma \\
\cL(\bu^{(l+1)}) \le  \cL(\bu^{(l)})  & \text{if } \bu^{(l)} \in \Gamma 
\end{cases}
\end{equation}
\item The algorithm $\cA$ is closed at points outside $\Gamma$.
\end{enumerate}
Then, every convergent subsequence of  $\{\bu^{(l)}\}_{l=0}^\infty$ converges to a solution of $\cA$. 
\end{theorem} 


\begin{lemma}\label{th:lemma}
Let us define the solution space $\Gamma$ as the set of stationary points of  \eqref{eq:dvrsty}.  
Then sequence $\{\bs^{(l)}\}_{l=0}^\infty$ generated by the proposed algorithm $\cA$ \eqref{eq:algorithm} satisfies the global convergence theorem.
\end{lemma}
\begin{proof}
We show that the conditions for the Theorem \ref{th:gconverge} hold for the sequence $\{\bs^{(l)}\}_{l=0}^\infty$ generated by the algorithm $\cA$.  \\
C1. Every closed and bounded set is compact. An equivalent condition is that the points in the sequence and its accumulation points are bounded. We can certainly find an upper bound of sequence $\{\bs^{(l)}\}_{l=0}^\infty$ generated by $\cA$. Such bounds exist as for a finite solution with finite points residuals cannot be arbitrary large. 
Moreover, the accumulation points are no greater than the bounding values. Therefore, such a compact subset $S$ can be constructed from the bounds. \\
C2. Given a real number $\gamma > 0$, 
 define $\cL(\bs):\cC\to \mathbb{R}$
\begin{equation}\label{eq:mainfunc}
\cL(\bs) = \sum_{i=1}^n \log (s_i + \gamma)
\end{equation}
where $\cC$ is the feasible region defined by the constraints in \eqref{eq:maxcon3}. 
For the points $\bs^{(l)} \not\in \Gamma$ 
\begin{equation}
\begin{aligned}
\frac{1}{n}\Big(\cL( \bs^{(l+1)}) - \cL(\bs^{(l)})\Big) 
~~~ = &  \sum_{i=1}^n \Big ( \frac{1}{n}\log (s_i^{(l+1)} + \gamma) - \frac{1}{n}\log (s_i^{(l)} + \gamma)\Big ) \\ 
~~~ = &  \sum_{i=1}^n \frac{1}{n}\log \frac{s_i^{(l+1)} + \gamma}{s_i^{(l)} + \gamma } <  \log \Big ( \frac{1}{n} \sum_{i=1}^n  \frac{s_i^{(l+1)} + \gamma}{s_i^{(l)} + \gamma } \Big ) \\ 
~~~ \le &  \log \Big ( \frac{1}{n} \sum_{i=1}^n  \frac{s_i^{(l)} + \gamma}{s_i^{(l)} + \gamma } \Big )  = 0 \Rightarrow \cL(\bs^{(l+1)}) < \cL(\bs^{(l)}). \nonumber
\end{aligned}
\end{equation}
Here the first inequality follows from the strict concavity property of the $\log (.)$ function. Note that the equality happens only when $\bs^{(l+1)} = \bs^{(l)}$ which implies $\left\langle \nabla f(\mathbf{u}^{(l)}),\; \mathbf{u} \right\rangle = 0$ (by eq. \eqref{eq:iter}). Thus the inequality is strict for $\bs^{(l)} \not\in \Gamma$. The second inequality follows from the fact that $\bs^{(l)}$ is obtained by minimizing $\sum_{i=1}^n  {s_i }/{(s_i^{(l)} + \gamma) } $, $\bs \in \cC$ and $\log (.)$ is monotonic increasing. Moreover, for $\bs^{(l)} \in \Gamma$ 
\begin{equation}
\begin{aligned}
~~~~ & \bs^{(l+1)} = \bs^{(l)} & \implies \cL(\bs^{(l+1)}) =  \cL( \bs^{(l)})~\\
\text{and ~~~~~ } & \bs^{(l+1)} \neq \bs^{(l)} & \implies \cL(\bs^{(l+1)}) <  \cL( \bs^{(l)}). 
\end{aligned}
\end{equation} 
Thus $\bs^{(l)} \in \Gamma$ implies $ \cL(\bs^{(l+1)}) \le  \cL( \bs^{(l)})$. \\ 
C3. A continuous mapping from a compact set to a set of real numbers is a closed map \cite{loring10}. The map $\cA$ is continuous and the set $S$, containing the elements of $\{\bs^{(l)}\}_{l=0}^\infty$, the range of the mapping $\cA$ in our algorithm, has already been proven as compact. \qed
\end{proof}  
%
  
\begin{theorem}\label{th:converge}
 For any starting point $\{\btheta^{(0)},\bs^{(0)}\} \in \cC$, there exist a subsequence of the sequence generated by \eqref{eq:algorithm} converges asymptotically to a stationary point  of \eqref{eq:dvrsty}. 
\end{theorem}
\begin{proof}
The sequence $\{\bs^{(l)}\}_{l=0}^\infty$ is compact. Therefore, there must exist a convergent subsequence $\{\bs^{(p_l)}\}_{l=0}^\infty$ of  $\{\bs^{(l)}\}_{l=0}^\infty$. By Lemma \ref{th:lemma}, the convergent subsequence $\{\bs^{(p_l)}\}_{l=0}^\infty$ converge to a stationary point of \eqref{eq:dvrsty}. 

\qed

%
\end{proof}
The above theorem shows that the objective of \eqref{eq:dvrsty} generated by the sequence $\{\btheta^{(l)},\bs^{(l)}\}_{l=0}^\infty$ strictly decreases and converges to a local minimum or a saddle point of \eqref{eq:dvrsty}. Further, by lemma \ref{lem:one}, any feasible solution of \eqref{eq:dvrsty} is also a local minimum of  \eqref{eq:maxcon3}. Thus, the proposed algorithm \eqref{eq:algorithm} is guaranteed to find a local minimum of \eqref{eq:maxcon3}. 



\section{Runtime Complexity}
The complexity of the proposed methods IR-LP depends on the complexity of the each iteration as maximum number of iterations $L$ is fixed. The global methods \cite{chincvpr2015} and \cite{olsson08} that require $\cO(k^{d+1})$ and $\cO \big((d+1)^k\big)$ number of iterations respectively, where $d$ is the dimension of the problem and $k$ is the number of outliers. Note that the above numbers are enormous compared to $L$ (choices of $L$ are discussed in results Section of the extended version). Further, in each iteration, those global methods solve a similar linear program or convex program. 
Furthermore, like \cite{chincvpr2015}, except the initial iteration, we initialize by the solution of the previous iteration. 

\emph{ Linear Residuals}
IR-LP solves a LP in each iteration which is remarkably efficient in practice. Moreover, as the coefficient matrix is extremely sparse, it becomes an effective solver \cite{boyd2004convex}.  
Although, there are worst-case polynomial time algorithms for solving a LP, \eg Karmakar's projective algorithm $\cO(n^{3.5}$), we utilize an approximate solution\footnote{{Since $\bs^{(l)}$ is only used to compute the weights $\bw^{(l+1)}$ in the next iteration, an approximate solution, which still minimizes the objective, is sufficient to initialize $\bs^{(l+1)}$.}}, which is solved in linear time \cite{megiddo1984linear}. 

\emph{ Quasiconvex Residuals}
IR-LP minimize linear objective under convex constraints that can be solved by an  interior point algorithm~\cite{ye1989extension} in polynomial time.

\section{Parameter Settings}\label{sec:parameters}
\subsection*{Initialization} \label{sec:initi}

The initialization of the shrinkage residuals $\bs^{(0)}$ can be aided using any fast approximate method. 
However, the initialization should not be too far from the optimal solution.  {In all of our experiments, unless stated otherwise, we initialize $\bs^{(0)} = \mathbf{1}$ and then iterate the first iteration to find a suboptimal solution $\btheta^{(1)}$. Again, $\btheta^{(1)}$ is utilized to update the shrinkage residuals $\bs^{(1)}$.  
 A better initialization (\texttt{RANSAC} solution or iterative $\ell_\infty$~\cite{sim06}) leads to a better solution in some cases, however, our chosen trivial initialization works  well in most of the applications. 
The results under different initializations are discussed in the extended version. 

\subsection*{Selecting $\gamma$}\label{sec:parachoice}

In the proposed algorithm, the constant $\gamma$ serves to bound the smooth objective from below, and also regularizes the optimization to avoid the stiffness to the solution where $s^{(l)}_i = 0$; intuitively, note that there will be points (\ie, the inliers) where the slack values are zero. In general, the algorithm works reasonably well with a small independent choice of $\gamma$. In this work, however, we chose $\gamma = 0.01$ for all the experiment reported and got satisfactory results. 

In the literature of reweighted methods, some works~\cite{candes2008enhancing,wipf2010iterative} exhibit better performance on some datasets by adapting $\gamma$. Specifically, \cite{candes2008enhancing} chose $\gamma^{(l+1)} = \max\{{\bs^{(l)}}^+, 0.01\}$ where $\bs^+$ are the positive slack variables, \cite{wipf2010iterative} utilized an annealing schedule and forced $\gamma^{(l+1)} \to 0$. However, note that for adaptively chosen $\gamma^{(l)}$, one can no longer guarantee the convergence of the algorithm.

\subsection*{Stopping Criterion}
Proposed iterative reweighted method IR-LP is executed till the objective function in two consecutive iteration is less than $\zeta$ or maximum number of iterations $L$ is exhausted. Now, if $\bs^{(l)}$ and $\bs^{(l+1)}$ are the shrinkage residuals of \eqref{eq:algorithm} in consecutive iterations, $\sum_{i = 1}^N w_i^{(l)} s_i^{(l)} - \sum_{i = 1}^N w_i^{(l)} s_i^{(l+1)} \ge 0$. We terminate the iteration once the difference is less than $\zeta$, \ie,  
\begin{equation}
\begin{aligned}
0 ~~\le~~ & ~~ \sum_{i = 1}^N {s_i^{(l)}}/({s_i^{(l)} + \gamma}) - \sum_{i = 1}^N {s_i^{(l+1)}}({s_i^{(l)} + \gamma}) && \le~~ \zeta~~~ \\ \nonumber
\Rightarrow 0 ~~\le~~ & \sum_{s_i^{(l)} >~ 0} ({s_i^{(l)} - s_i^{(l+1)}})/({s_i^{(l)}/\gamma + 1}) - \sum_{s_i^{(l)} =~ 0} {s_i^{(l+1)}} && \le~~ \gamma \zeta~~~ 
\end{aligned}
\end{equation} 
Thus for a smaller value of $\gamma$, the above constraint enforces a small variability of $\bs^{(l+1)}$. Notice that $\gamma$ is not involved for the inlier residuals in the above expression. Thus, a small number of iteration $L$ is required for a small choice of $\gamma$. However, in practice with the above choice of $\gamma$, the proposed method works quite well with $L = 25$  and $\zeta = 10^{-4}$.

\section{Results}\label{sec:results}

To evaluate the proposed method IR-LP, a number of experiments have been performed on synthetic and real datasets. We compared IR-LP against state-of-the-art approximate methods for MaxCon, namely
\begin{itemize}
\setlength\itemsep{0em}
\item IR-QP: a reweighted least square scheme obtained by replacing each iteration of \eqref{eq:algorithm} by a quadratic program (QP) under linear or quasiconvex residuals (described in the extended version). Note that  there is no closed form solution of each iteration and one needs to solve a convex quadratic program. 
\item Olsson \etal's $\ell_1$ method~\cite{olsson10}; see~\eqref{eq:maxcon4}.
\item {Sim and Hartley's $\ell_\infty$ method~\cite{sim06}, where the $\ell_\infty$ is recursively solved and the data with the largest residuals are removed from the subsequent iterations. }
\item As a baseline, we ran vanilla \texttt{RANSAC} with confidence $\rho = 0.99$~\cite{fischler1981random}. 
\item \texttt{MLEsac} method~\cite{torr2000mlesac}, that adopts similar sampling strategy as \texttt{RANSAC} to instantiate models, but chooses the one that maximizes the likelihood. 
 
\item We also run locally optimize \texttt{LO-RANSAC} \cite{lebeda2012fixing} as a baseline. We only run our own implementation where the parameters were carefully chosen from Table 1 of  \cite{lebeda2012fixing}. The stopping criterion was considered same as vanilla \texttt{RANSAC}. 
\item For the experiments with real data, we also consider \texttt{L-RANSAC} -- allowing vanilla \texttt{RANSAC} to run same amount of time as the proposed method IR-LP. 
\item  We also execute a global method \texttt{ASTAR}~\cite{chincvpr2015}{\footnote{\url{http://pulakpurkait.com/Data/astar_cvpr15_code.zip}}} with maximum allowable runtime $300$ seconds. Note that as the global method is terminated early, it cannot guarantee optimality. 

\end{itemize}
All the methods were implemented in {Matlab} and executed on a $i7$ $CPU$.


\begin{figure} \center 
\subfigure[Average consensus size found.]{\includegraphics[width=0.49\columnwidth]{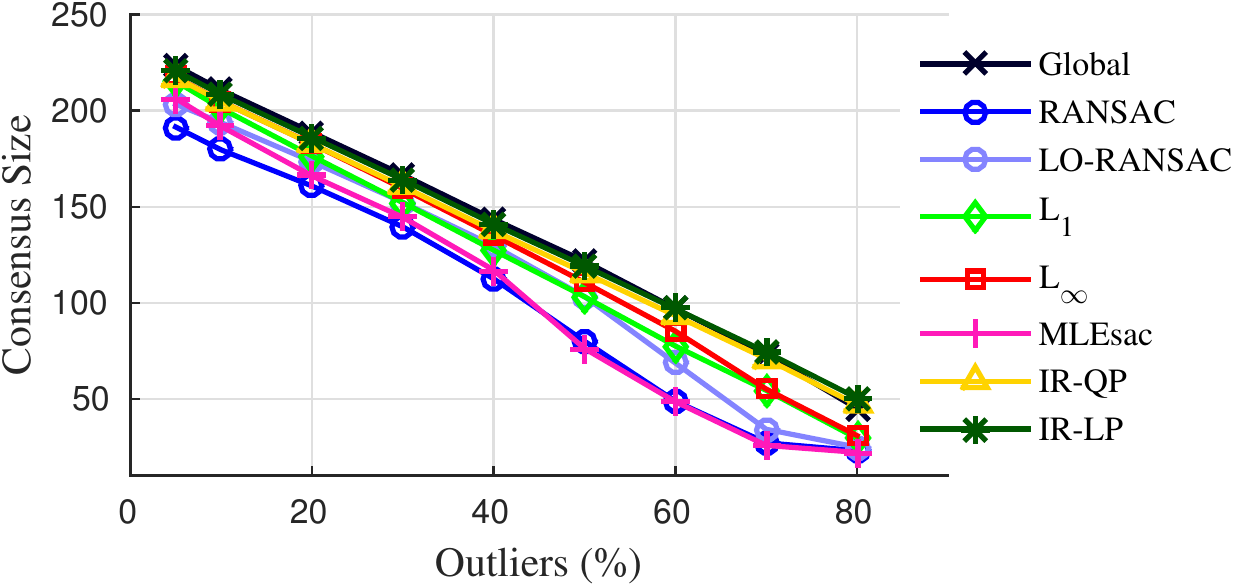}\label{fig:valueline}} \hspace{-0.0cm}
\subfigure[$\log$ of the average run time (seconds).]{\includegraphics[width=0.49\columnwidth]{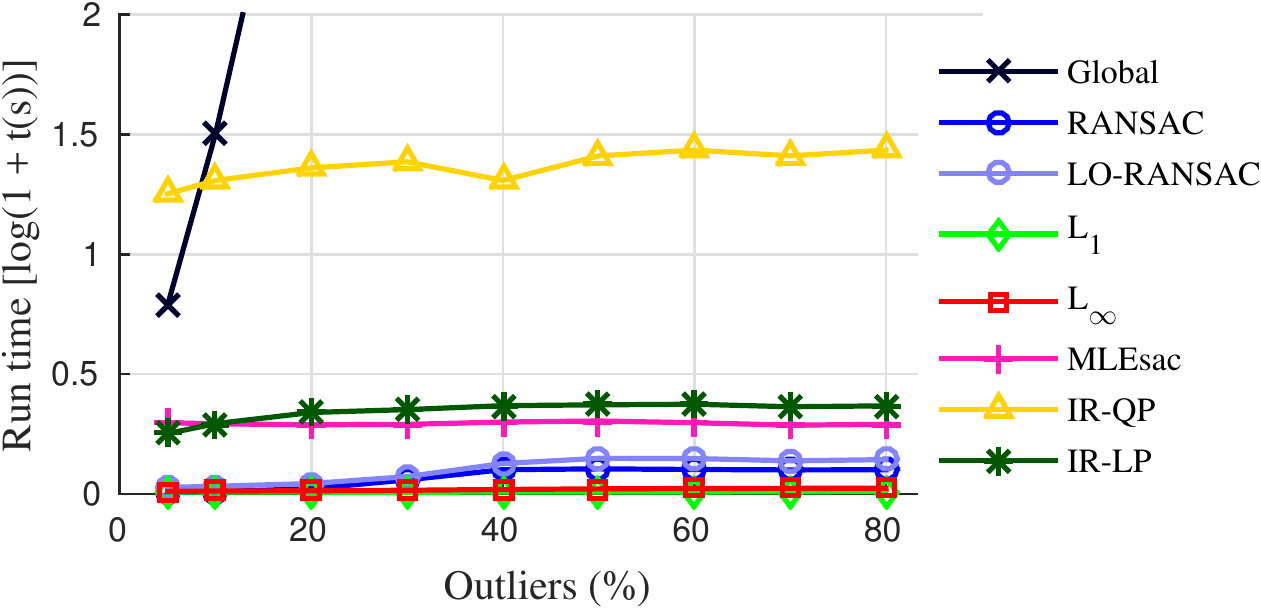}\label{fig:runtime}} \\ 
\subfigure[a zoomed and cropped version of above.]{\includegraphics[width=0.60\columnwidth]{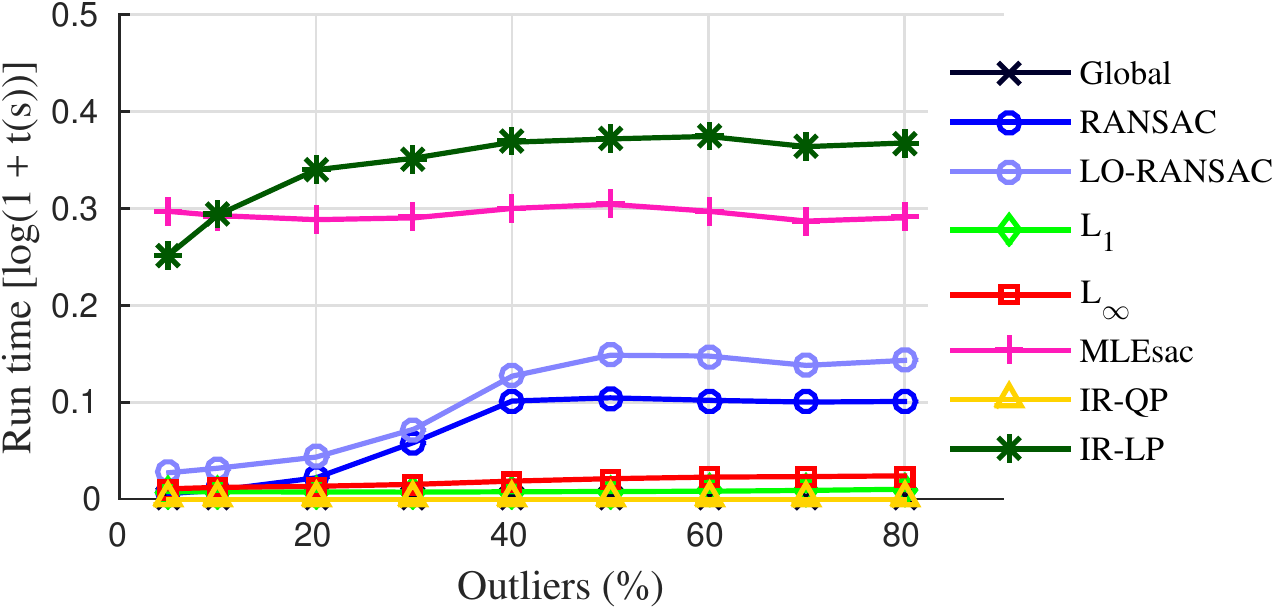}\label{fig:runtime2}} 
\caption{Hyperplane fitting results. Proposed IR-LP clearly dominates the other methods. Please see text for details.}
\label{fig:lineComparison}
\end{figure}

Note that when $r(\btheta;\bx_i)$ is linear, the subproblems (each iterations) of $\ell_1$ and proposed IR-LP are LPs, while for IR-QP the subproblems are QPs. The optimization toolboxes $\ell_1$-magic\footnote{\url{http://statweb.stanford.edu/~candes/l1magic/}}~\cite{candes2005decoding} and \texttt{cvx}\footnote{\url{http://cvxr.com/cvx/}}~\cite{grant2008cvx} are employed to solve the LPs and QPs. 
When $r(\btheta;\bx_i)$ is quasiconvex, the subproblems of all the methods are convex programs~\cite{boyd2004convex}; we solved each convex program instances again with \texttt{cvx}. 


\subsection{Hyperplane fitting} \label{sec:hplane_fitting} 
We generated {$N = 250$} points around an $8$-dimensional hyperplane under Gaussian noise with $\sigma_{in} = 0.1$. A number of the points ($5\%$--$80\%$) were then corrupted by a uniform noise (interval $[-10, 10]$) to simulate outliers. The inlier threshold  was chosen as $\epsilon = 0.3$. For a chosen outlier percentage, we generated $100$ instances of the data and ran the different methods. Figures \ref{fig:valueline} and \ref{fig:runtime} show the average consensus size and run time over the synthesized data. 

While $\ell_1$, $\ell_\infty$ and \texttt{RANSAC} were very fast, they usually produced lower quality results, in terms of the discrepancy with the global solution. While the solution quality of IR-LQ was better to $\ell_1$ and \texttt{RANSAC}, it was much slower, owing to the fact that a QP needs to be solved in each iteration. \texttt{MLEsac} is slower than other randomized method as it has an additional inner loop to estimate the mixing parameter. Further, unlike \texttt{RANSAC}, no probabilistic bounds for number of iterations has been incorporated for \texttt{MLEsac} and executed for $500$ iterations. However, as \texttt{MLEsac} has different criterion (ML) for model estimation, it produces no better solution than other suboptimal methods. \texttt{LO-RANSAC} performs quite well for low outlier ratio. It is clear from the figures that proposed IR-LP was able to produce near optimal solutions in all the cases; in fact, we observed that IR-LP produced optimal solutions in almost $30\%$ of the runs. 
Furthermore, the proposed IR-LP is most effective for the cases with ($50\%$--$70\%$) outlier ratio which are the most common scenarios for the real datasets. 

\newcolumntype{g}{>{\columncolor{gray}}r} 
\begin{sidewaystable}\setlength{\tabcolsep}{3pt}\renewcommand{\arraystretch}{1.3}\setlength{\tabcolsep}{3pt} 
\scriptsize
\begin{center}
  
  \begin{tabular}{rp{1.5cm}r|rg|rg|rg|rg|rg|rg|rg|rg|rg} 
  \hline   
  & \multicolumn{2}{c|}{Methods}  & \multicolumn{2}{c}{\texttt{ASTAR}\cite{chincvpr2015}} & \multicolumn{2}{|c}{\texttt{RANSAC}}  & \multicolumn{2}{|c}{$\ell_1$ \cite{olsson10} } & \multicolumn{2}{|c}{$L_\infty$ \cite{sim06}} & \multicolumn{2}{|c}{\texttt{MLEsac} \cite{torr2000mlesac}} & \multicolumn{2}{|c}{\texttt{LO-RANSAC}\cite{lebeda2012fixing}} & \multicolumn{2}{|c}{IR-QP}   & \multicolumn{2}{|c}{IR-LP} & \multicolumn{2}{|c}{\texttt{L-RANSAC}}\\ \hline \hline 
  &   Datasets    			&   $n$    & $|\cI^*|$   & t (s)  & $|\cI|\pm \sigma$ ~~   & t (s)   & $|\cI|$           & t (s) & $|\cI|$         & t (s) & $|\cI|\pm \sigma$ ~~          & t (s) & $|\cI|\pm \sigma$~ 	 & t (s)     & $|\cI|$  	& t (s)     & $|\cI|$ 		& t (s)   & $|\cI|\pm\sigma$~ & t (s) \\  \hline 
 \multirow{6}{*}{\rotatebox[origin=c]{90}{linear}}   &   Valbonne Ch. & 108      		& {\bf 67}      & 300  & 64.5$~\pm$~~3.1      & 0.02    & 26     		&     0.01     & 37       		& 0.02		& 62.1$~\pm$~~2.4       	& 0.61		& 66.8$~\pm$0.7  &0.06	 & 61 			& 4.71		    & {\bf 67}        	& 0.07 & 64.8$~\pm$0.3		& 0.07 \\ 
   &   University Lib.  & 665         & 552  	&  300     & 542.3$~\pm$12.1 	&  {\bf  0.03}     & 435  		    & 	0.03        & 251		    & 1.30	& 523.4$~\pm$18.7		    & 1.34		&546.8$~\pm$3.8 & 0.27			& 550 			& 5.60		   & {\bf 553}          &  0.21 & 546.2$~\pm$1.8	 &  0.21\\ 
    &  Keble College  		& 399	& {\bf 311} 	& 300  & 305.7$~\pm$~~6.0 	& {\bf 0.03}   & 102 			& 0.02		      & 145 			& 0.10   	& 224.3$~\pm$26.8		& 1.48   	&308.0$~\pm$0.5  &0.13		& 310 			& 2.37		   &  {\bf 311} 		   & 0.07  & 307.5$~\pm$0.9	 & 0.07  \\
     
 & Road Sign  		& 31 	& {\bf 29} 	& 	300 	& 	27.6$~\pm$~~2.1 	& {\bf  0.01}   	&    14 	& 0.01 &  22 			& 0.01 	& 27.3$~\pm$~~0.6 &  0.53 & {\bf 29.0$~\pm$0.0}  & {\bf 0.01} & {\bf 29} & 0.71 & {\bf 29}  & {\bf 0.01} & 28.9$~\pm$0.0 & {\bf 0.01} \\ 

& House  		& 492	& {\bf 355} 	& 300 	& 351.0$~\pm$~~7.8 	& 0.03 	&  261	& {\bf 0.04} &  194	& 0.21  & 	344.1$~\pm$~~9.5	& 1.60	 & 349.1$~\pm$2.0  & 0.12 & 351 & 3.14 & 352 & 0.24 & 353.2$~\pm$2.5 &  0.24 \\ 
  &  Cathedral  		& 544	& {\bf 481} 	& 300    & 464.5$~\pm$13.2  	& {\bf 0.02}   & 445	& 0.03		      & 289		& 0.27    	& 465.7$~\pm$12.5		& 1.84   	& 473.4$~\pm$5.1		& 0.22			& 479  			& 7.03		   & 479		   & 0.36  & 470.5$~\pm$0.8	 & 0.36  \\ \hline  
           &   Datasets    			&   $n$    & $|\cI^*|$   & t (s)  & $|\cI|\pm \sigma$ ~~   & t (s)   & $|\cI|$           & t (s) & $|\cI|$         & t (s) & $|\cI|\pm \sigma$ ~~          & t (s) & $|\cI|\pm \sigma$~ 	 & t (s)     & $|\cI|$  	& t (s)     & $|\cI|$ 		& t (s)   & $|\cI|\pm\sigma$~ & t (s) \\  \hline 
 \multirow{6}{*}{\rotatebox[origin=c]{90}{~~~~quasiconvex~}}   &         Valbonne Ch. & 108  		& 83      & 21.4  & 75.6$~\pm$~~6.1      & 0.02  & 26     		&  0.19     & 60       	& 1.22      & 71.1$~\pm$~~4.7       	& 1.11    & 82.3$~\pm$2.4 & 0.13    & {\bf 83}   		& 6.61    & 83        & 6.23 & 83.9$~\pm$0.1 & 6.23    \\ 
       &   University Lib.  & 665		& 598  	&  300 & 590.5$~\pm$19.9  	& {\bf  0.03}  & 464  		    & 	2.74    & 338		    & 8.42   	 & 529.9$~\pm$16.7	  & 1.93    & 601.9$~\pm$2.3	    & 0.31   		   	& 608 	& 	106.73 & {\bf  613}        &  25.72   & 606.9$~\pm$0.1 &  25.72 \\ 
       &   Keble College  		& 399	 &  309 	& 300  & 306.1$~\pm$~~5.1 	& {\bf 0.02} &  92 		& 	0.64	& 177 		& 2.13    & 301.9$~\pm$~~2.0 		& 	2.07	 & 307.8$~\pm$1.3	  & 0.14  	 & 303		   & 32.87   & 308	  & 7.06 & {\bf 309.8$~\pm$0.4}  &	7.06\\ 
    &   Road Sign  		& 31 	& 29 	& 300 & 28.5$~\pm$~~4.1 			&    {\bf  0.01}	 		& 2   & 0.558 	 &	 23 		& 	0.242  & 28.7$~\pm$~~0.4& 1.39& 28.3$~\pm$0.8 & 0.03 & 30 & 3.92 & {\bf 30}  & 1.63 &30.0$~\pm$0.0 & 1.63 \\ 
      
    &   House  		& 492	& 349 	& 300 & 353.0$~\pm$~~8.2 			& {\bf  0.06} 	 	& 273   & 1.922   	 & 277 		& 	1.278  & 352.8$~\pm$~~0.7 & 3.37 &	353.0$~\pm$~0.0 & 0.12 & 292 & 36.73 & {\bf 355}  & 24.82 & 354.0$~\pm$0.0   & 24.82 \\  
        
     &   Cathedral  		& 544	 & 473 	& 300   	& 463.0$~\pm$14.7 	& {\bf  0.02} 	& 461	& 1.23 	    &	438   & 1.28 &	468.0$~\pm$~~4.9   & 1.72   & 473.7$~\pm$4.7	  & 0.19  & 471 & 28.50 &  {\bf ~~481} & 9.03 & 479.8$~\pm$0.1  & 9.03\\ \hline  \vspace{0.2cm}\\ \hline   & \multicolumn{2}{c|}{Methods}  & \multicolumn{2}{c}{\texttt{ASTAR}\cite{chincvpr2015}} & \multicolumn{2}{|c}{\texttt{RANSAC}}  & \multicolumn{2}{|c}{$\ell_1$ \cite{olsson10}} & \multicolumn{2}{|c}{$L_\infty$ \cite{sim06}} & \multicolumn{2}{|c}{\texttt{MLEsac} \cite{torr2000mlesac}} & \multicolumn{2}{|c}{\texttt{LO-RANSAC}\cite{lebeda2012fixing}} & \multicolumn{2}{|c}{IR-QP}   & \multicolumn{2}{|c}{IR-LP} & \multicolumn{2}{|c}{\texttt{L-RANSAC}}\\ \hline \hline 
    
     \multirow{7}{*}{\rotatebox[origin=c]{90}{linear}} &   Datasets    				&   n  	& $|\cI^*|$  	& t (s)  & $|\cI|\pm \sigma$~~~    		& t (s)  & $|\cI|$    		& t (s)   	& $|\cI|$    		& t (s)   	& $|\cI|\pm \sigma$~~~           & t (s)  	& $|\cI|\pm \sigma$~~~  	& t (s)     & $|\cI|$  	& t (s)     & $|\cI|$  		& t (s) & $|\cI|\pm\sigma$~~~ & t (s) \\ \hline 
 	&    Valbonne Ch. 			& 108  	 & {\bf 88} 	& 300	& 77.1$~\pm$~~2.8 			& {\bf  0.06}   & 17			&	0.01      & 65   		&  0.02  & 73.9$~\pm$~~4.1   		&  0.78 		 & 80.6$~\pm$~~2.5   		&  0.14 			 & 78   		&  1.98			& 85 			& 0.18 & 80.1$~\pm$~~1.7  & 0.18 \\ 
 	&    Wadham Cl.  			& 1051	& {\bf 365}  	& 300       & 287.8$~\pm$18.1  	& {\bf  0.18}  & 129  	    & 	0.02  	& 213		    & 0.09   	   	& 242.5$~\pm$29.5 		    &    1.15 			& 317.0$~\pm$21.3 	    &    0.22 			& 312	& 13.19 		   & 344        &  0.38	&	307.8$~\pm$12.9 &  0.38	 \\ 
 	 & M. College I  			& 577		& {\bf 234} 	& 300 	& 212.6$~\pm$~~3.6 		& {\bf 0.09}  	   	&   79 		&  0.01		     	& 58  		& 0.03 & 197.2$~\pm$~~4.9 &	0.91 &  212.8$~\pm$~~1.2		&    0.16		& 211& 10.83	& 207 	& 0.13 &  216.1$~\pm$18.1 	& 0.18 	\\ 
	&     Merton Cl. III  			& 313		& {\bf 214} 		& 300	& 176.8$~\pm$~~6.8		& {\bf 0.06}   & 44 		&  0.01		   		   	& 174  		& 0.01   	&  155.4$~\pm$~~7.2		&    0.59		&  184.9$~\pm$~~7.0 	&    0.08		& 197  		& 	7.23	& 210 		& 0.24  & 189.3$~\pm$~~2.6		&  0.24	\\ 
  	&   Corridor  				& 124		& {\bf 72} 		& 300		& 55.3$~\pm$~~2.5 		& {\bf  0.11}   & 13 		&  0.01		    & 55  		& 0.01  		& 	41.9$~\pm$~~3.9		&   	0.32		& 	59.4$~\pm$~~2.8 	&   	0.12		& 56  		& 5.39		& 67 		& 0.16 & 57.6$~\pm$~~0.8 	& 0.16 \\  
  	& Dinosaur  		& 156		& {\bf 94} 		& 300 	& 68.6$~\pm$~~5.3 		&  0.08   	&  85 	   &  0.11	    &  	23	        &  0.02 	& 33.8$~\pm$~~3.7 			&  0.04   				& 70.0$~\pm$~~3.9 	& 0.01 	& 92 & 	0.04 & 	78		&   	{\bf 0.07}	 & 82.3$~\pm$~~2.7 	&  0.09  \\ \hline   

\end{tabular}
\end{center}
\caption{First two blocks: Results for homography estimation with linear and quasiconvex residuals. The last block: Results for linearised fundamental matrix estimation. $n$: number of point correspondences, $|\cI|$: consensus set size (average for the randomized methods), $\sigma$: $std$ of the consensus set size, $|\cI^*|$: optimal consensus set size, t(s): runtime in seconds. The columns corresponding to the runtime are marked by gray, and the best values (the maximum size consensus set and the runtime) are marked with bold fonts. }
\label{tab:linerHomography}  
\end{sidewaystable}

\subsection{Homography fitting}\label{res:homo}
In this experiment, we used images from the Oxford Visual Geometry Group\footnote{\url{http://www.robots.ox.ac.uk/~vgg/data/}}, namely, Valbonne Church (image index 4 and 7), University Library (image index 1 and 2), and Keble College (image index 2 and 3). These images have been used extensively in previous works on geometric estimation. On each image pair, SIFT key-points were detected and matched using the \texttt{VLFeat} toolbox\footnote{\url{http://www.vlfeat.org}}, where the second nearest neighbour test was invoked to prune wrong matches. We used the default parameters in \texttt{VLFeat}.
Both linearised residuals and geometric (quasiconvex) residuals are considered for homography estimation, which involves estimation of an $8$D parameter vector $\btheta$. 

\paragraph{ Linearised residuals}
The reader is referred to \cite[Section~4.1.2]{hartley2003multiple} on linearising the residuals for homography estimation. Each point-sets  were normalized separately by translating to mean $=0$ and scaling to $std =\sqrt{2}$. The inlier threshold $\epsilon$ was chosen as $\epsilon = 0.1$. Table~\ref{tab:linerHomography} presents the results of all methods. For \texttt{RANSAC} and other randomized methods, the results were averaged over $100$ runs. While $\ell_1$ was very fast, its solution quality was very poor --- this was most likely because the distribution of outliers in real data is not balanced, unlike in synthetic data where the outliers were considered to be uniformly distributed. It can also be seen that IR-QP is much slower than the other methods. We executed an efficient implementation of \texttt{LO-RANSAC}, but we believe, it has similar runtime complexity as \texttt{RANSAC}. In contrast, IR-LP always  produces larger size consensus set, and while its runtime was longer than \texttt{RANSAC}, \texttt{LO-RANSAC} and $\ell_1$, it was much faster than IR-QP. This proves overall better performance for IR-LP. 

\paragraph{ Quasiconvex residuals}\label{sec:quasi}
Model estimation under quasiconvex residuals is more geometrically meaningful, and inlier thresholds can be quoted in geometric units (pixels). The reader is referred to~\cite{kahl08} for the precise formulation of quasiconvex residuals for homography estimation.

Results under the inlier threshold $\epsilon = 1$ pixels are shown in Table \ref{tab:linerHomography}. On average proposed IR-LP managed to return the approximate solution that is better than the other methods. Both IR-QP and IR-LP were able to significantly improved upon the other methods, and the final solution quality of IR-QP/IR-LP were much higher than iterative $\ell_1$ and $\ell_\infty$. Under quasiconvex residuals, IR-LP is equally expensive as IR-QP due to the requirement of solving convex programs.

\subsection{Fundamental matrix estimation} 
We repeat the previous experiment, for linearised fundamental matrix estimation, on the same set of image pairs. Refer to \cite[Section~9.2.3]{hartley2003multiple} for the precise procedure in linearising the residual for fundamental matrix estimation. The normalizations of the individual point-sets were also performed here. $\btheta$ is also $8$-dimensional and inlier threshold $\epsilon$ was  chosen to be $0.1$. To test the optimum performance of all methods, we did not enforce the rank-$2$ constraint on the resulting fundamental matrices in all the methods. 

{We observe that a simple choice of the  initialization $\bs^{(0)} = \mathbf{1}$ does not  lead to a satisfactory local solution for this experiment. Here we initialize $\btheta$ by the solution of the iterative $\ell_\infty$ algorithm~\cite{sim06}} $\btheta_\infty$. The shrinkage residuals $\bs^{(0)}$ for all the points are then computed by evaluating residuals at $\btheta_\infty$. The \texttt{RANSAC} solution could also be another choice for initialization. However, iterative $\ell_\infty$ was chosen purely on computational basis. The results of different methods are shown in Table~\ref{tab:linerHomography}. The runtime for the iterative $\ell_\infty$ is added with the runtime of IR-LP and IR-QP. As the  iterative $\ell_\infty$ method is very fast, its local refinement  by proposed method is an attractive choice for fundamental matrix estimation. 

\section{Conclusions}
In this work, we formulated the maximization of the size of a consensus set as the iterative minimization of the re-weighted $\ell_1$ norm of the shrinkage residuals. Then, we illustrated different smooth surrogates of MaxCon. Followed by the minimization of a smooth surrogate that led to an iterative reweighted algorithm IR-LP. A convergent analysis and the runtime complexity of IR-LP are also discussed. 
Furthermore, a number of reweighted methods is  derived for this task and compared with the proposed method. Experimental results show the efficiency of the proposed method compared to the existing approximate methods.  Finally, we would like to draw an attention to the fact  that, in the linear residual case, each iteration of our algorithm simply requires solving a single LP, and thus the method can be implemented very easily using the existing optimization tools. Thus, our method can surely be used as a replacement of the randomized methods. 

 \newpage
\noindent {\Large \bf Supplementary Material: } 

\section{MaxCon - minimizing diversity of residuals}
In this section, we derive the connection between the Maximum Consensus problem and the Majorization-Minimization (MM) algorithm\footnote{As the current section  address some insights of the proposed method, only interested readers are encouraged to go through this section, others are redirected to the results section~\ref{sec:results} for more results}. 
The MaxCon can be written as follows:
\begin{equation} \label{equ:maxcon3}
\begin{aligned} 
\min_{\btheta,\; \bs} &
~~~ \sum_{i=1}^n \mathbf{1}(s_i ), &\\ 
~~\text{subject to} & 
~~~ r(\btheta; \bx_i) \leq \epsilon + s_i, & s_i \ge 0,\\ 
\end{aligned} 
\end{equation}
where $ \mathbf{1}(s_i)$ is an indicator function that returns $1$ if $s_i$ is \emph{non-zero}. 
The convex relaxation to~\eqref{equ:maxcon3} is the minimization of  absolute sum of the shrinkage residuals
\begin{align}\label{equ:maxcon4} 
\begin{aligned}
& ~~~~~~~~\min_{\btheta,\; \bs}
& &  \sum_{i=1}^n s_i &\\
& \text{subject to}
& & r(\btheta; \bx_i) \leq \epsilon + s_i, & s_i \ge 0, \\
\end{aligned}
\end{align}
which is also a robust estimation of the model parameters $\btheta$. 

\subsection{Minimizing diversity of residuals} \label{sec:iterirl1}

The difference between the objective of \eqref{equ:maxcon3} and \eqref{equ:maxcon4} is in how they ``count" the coefficients $\bs$,  affects the magnitude of the optimized $\bs$. Specifically, the larger coefficients are penalized more heavily in \eqref{equ:maxcon4} than smaller coefficients, unlike in \eqref{equ:maxcon3} where positive magnitudes are penalized equally. Intuitively, therefore, in the solution of \eqref{equ:maxcon3}, the shrinkage residuals will be less diverse (more concentrated) than the shrinkage residuals in the solution of \eqref{equ:maxcon4}. 

We demonstrate this observation in Figure \ref{fig:majorizationplot}, where we consider a line fitting problem. The solutions of Maxcon and a suboptimal solution are plotted along with the histogram of the optimized shrinkage residuals. Clearly the shrinkage residuals corresponding to the MaxCon solution are less diverse. This motivated us to seek a representation that aims to minimize diversity among the shrinkage residuals, with the ambition that it would lead to the MaxCon solution \eqref{equ:maxcon3}.

\begin{figure}[t]\centering
\subfigure[]{ \includegraphics[width=0.42\columnwidth]{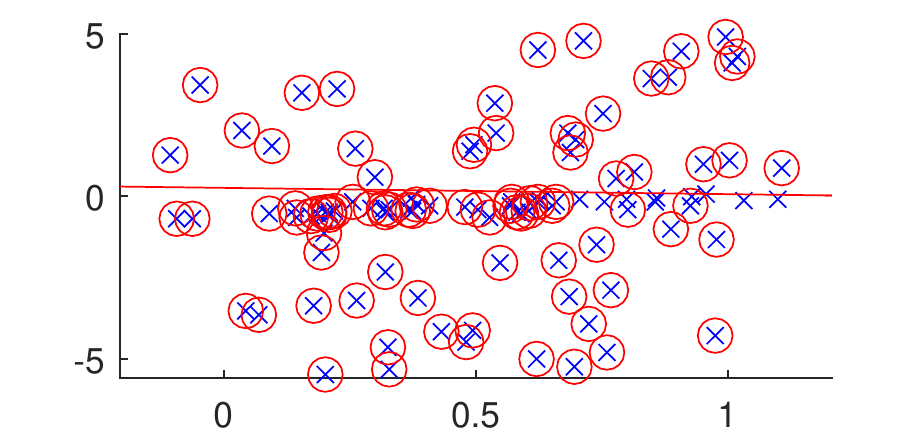}\label{fig:majorizationplot1}}
\subfigure[]{\includegraphics[width=0.42\columnwidth]{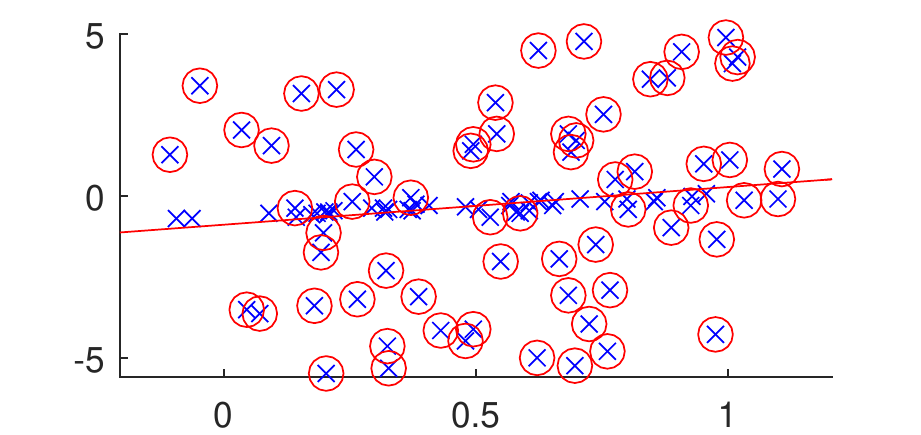}\label{fig:majorizationplot2}} \\ 
\subfigure[]{\includegraphics[width=0.42\columnwidth]{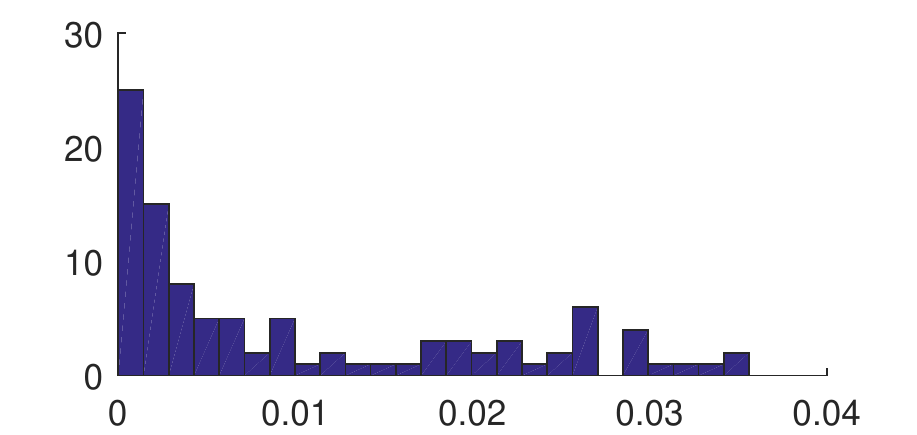}\label{fig:majorizationplot3}}
\subfigure[]{\includegraphics[width=0.42\columnwidth]{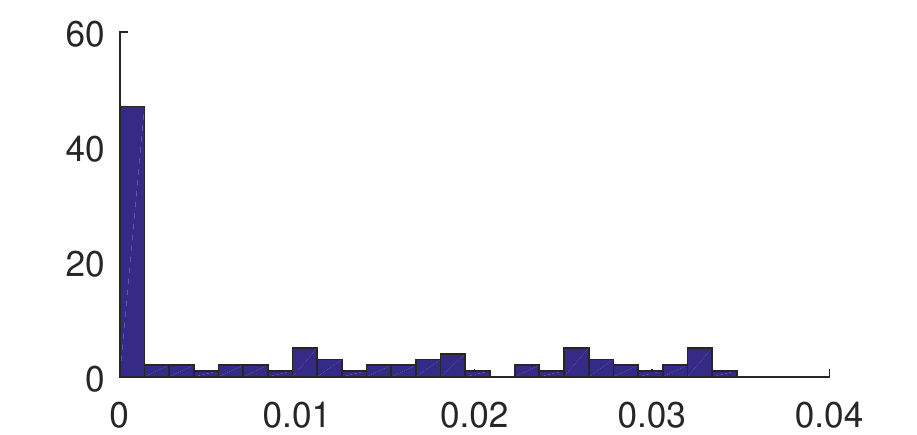}\label{fig:majorizationplot4}}
\caption{Line fitting problem: (a) and (b) are the MaxCon solution and a suboptimal solution. Circled data are those with non-zero shrinkage residuals, i.e., the identified outliers; (c) and (d) are the corresponding histograms of the non-zero shrinkage residuals.}
\label{fig:majorizationplot}
\end{figure}


\subsection{Majorization and Schur-concavity}

We follow the same notations and symbols used in \cite{marshall2010inequalities} to develop the background on majorization. 
\begin{definition}\label{dfn:majorization}
A preordering $\prec$ on the non-negative orthant $\mathbb{R}^n_+$ is defined for $\bs, \bt \in \mathbb{R}^n_+ \subset \mathbb{R}^n$ by 
\begin{align*}
\bs \prec \bt & \text{ if} & \begin{cases} \sum_{i = 1}^{k} s_{\lfloor i \rfloor} \le \sum_{i = 1}^{k} t_{\lfloor i \rfloor}, k = 1, \ldots, n - 1\\ \sum_{i = 1}^{n} s_{\lfloor i \rfloor} = \sum_{i = 1}^{n} t_{\lfloor i \rfloor} \end{cases}
\end{align*}
where $s_{\lfloor i \rfloor}$ and $t_{\lfloor i \rfloor}$ denotes the non-increasing arrangements\footnote{i.e., an arrangements of elements of the vector $\bs$, so that $s_{\lfloor 1 \rfloor} \ge s_{\lfloor 2 \rfloor} \ge s_{\lfloor 2 \rfloor} \ldots \ge s_{\lfloor n \rfloor}$} of the elements of $\bs$ and $\bt$. We say $\bs$ is {\bf majorized} by $\bt$ if $\bs \prec \bt$. 
\end{definition}

When $\bs \prec \bt$, $\bs$ is more diverse than $\bt$ or, equivalently $\bt$ is more concentrated than $\bs$. Let us denote the sequence of partial sums by $S_s \lfloor k \rfloor$, i.e., $S_s \lfloor k \rfloor = \sum_{i = 1}^{k} s_{\lfloor i \rfloor}$. Then the majorization order can also be rewritten as 
\begin{align*}
\bs \prec \bt & \text{ if } & \begin{cases}  S_s \lfloor k \rfloor \le S_t \lfloor k \rfloor, k = 1, \ldots, n - 1 \\ S_s \lfloor n \rfloor = S_t \lfloor n \rfloor \end{cases} 
\end{align*}
The \emph{Lorentz curve} is a plot of $S_s \lfloor k \rfloor$ against $k$. Clearly, if the Lorentz curve of $S_s \lfloor k \rfloor$ lies under the Lorentz curve of $S_t \lfloor k \rfloor$ everywhere, then $\bs \prec \bt$. Two vectors cannot be related by the majorization if the corresponding Lorentz curves intersect. In Fig. \ref{fig:lorentz}, we demonstrate the  properties of Lorentz curves.  

It can be easily proved that the preorder $\prec$ defined above is also a partial order relation. \ie $\prec$ is not only reflexive and transitive but also antisymmetric.\\

\begin{figure}\centering
\includegraphics[width=0.5\columnwidth]{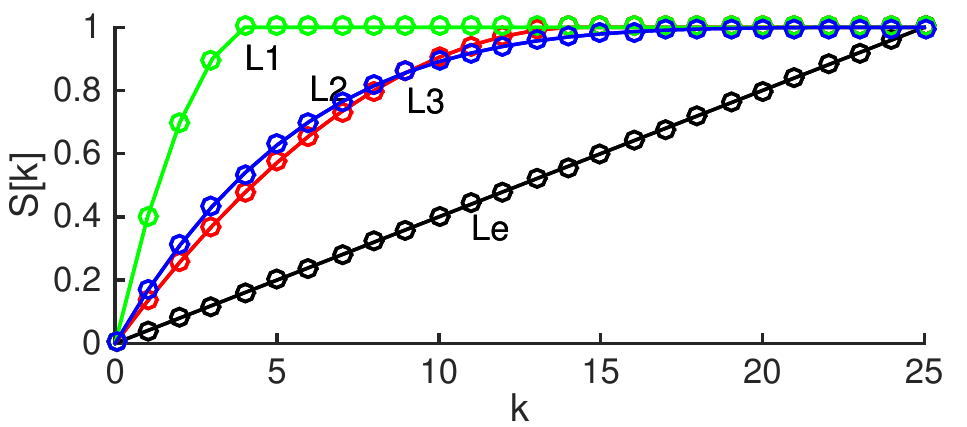}
\caption{Different Lorentz Curves  for a population of size $n = 25$. The curve corresponds to maximum diversity is $Le$. $L1, L2$ and $L3$ curves correspond to the vectors $\bs_1, \bs_2$ and $\bs_3$ where $\bs_1 \prec \bs_2$ and $\bs_1 \prec \bs_3$, i.e.,  $\bs_1$ represents the minimum diversity among them. Since the curves $L2$ and $ L3$ intersect, the corresponding vectors $\bs_2$ and $\bs_3$ cannot be ordered by majorization.}
\label{fig:lorentz}
\end{figure}

We prove the following theorem which relates the key ideas of the current work. 
\begin{theorem}\label{th:mainth} 
Let $\bs $ and $ \bt$ be the shrinkage residuals corresponding to two different solutions of \eqref{equ:maxcon3}. If $\bs \prec \bt$, then the number of inliers of the solution corresponding to $\bt$ is greater than or equal to the number of inliers of the solution corresponding to $\bs$. The converse is not generally true. 
\end{theorem}
\begin{proof}
Let us assume that the solution corresponding to $\bs$ contains more inliers than the solution corresponding to $\bt$ while $\bs \prec \bt$, i.e., 
\begin{equation}\label{equ:thm1}
 \sum_{i=1}^n \mathbf{1}(s_i = 0) > \sum_{i=1}^n \mathbf{1}(t_i = 0)
\end{equation}
 as inliers corresponds to coefficients $s_i = 0$. Let $ c = \sum_{i=1}^n \mathbf{1}(s_i = 0) $ then the above implies 
\begin{align}\label{equ:thm2}
 0 = \sum_{i=n-c+1}^{n} s_{\lfloor i \rfloor} < \sum_{i=n-c+1}^{n} t_{\lfloor i \rfloor}. 
\end{align}
 Further assume that $S_s $ and $S_t $ are normalized into sum to one. i.e. 
\begin{equation}\label{equ:thm3}
1 = \sum_{i=1}^{n} s_{\lfloor i \rfloor} = \sum_{i=1}^{n} t_{\lfloor i \rfloor}.
\end{equation} 
 Then subtracting \eqref{equ:thm2} from \eqref{equ:thm3}, 
 \begin{align}
 1 = \sum_{i=1}^{n-c} s_{\lfloor i \rfloor} > \sum_{i=1}^{n-c} t_{\lfloor i \rfloor},
 \end{align}
which contradicts the Definition \ref{dfn:majorization} for $\bs \prec \bt$. 

Conversely, for the case when the solution corresponding to $\bt$ contains more inliers than the solution corresponding to $\bs$ and the respective Lorentz curves $Lt~\&~Ls$ intersects, then $\bs~\&~\bt$ are not related by majorization order. \qed
\end{proof}

The theorem above effectively says that the MaxCon solution is the least diverse among all possible set of residuals that related by the partial order $\prec$. \ie, if we could minimize the diversity over the constraints in \eqref{equ:maxcon3}, hopefully, we end up with the MaxCon solution. 


\begin{definition}
A function $\phi:\mathbb{R}^n_+ \rightarrow \mathbb{R} $ is said to be {\bf Schur-concave} if $\phi (\bs) \ge \phi (\bt)$ whenever $\bs \prec \bt$ and strictly Schur-concave if in addition $\phi(\bs) > \phi(\bt)$ when $\bs$ is not a permutation of $\bt$. 
\end{definition}

\begin{theorem}\label{th:drvtest}
Let $I \subset \mathbb{R}$ be an open interval and let the function $\phi: I^n \rightarrow \mathbb{R}^n $ be continuously differentiable. Then $\phi $ is Schur-Concave on $I^n$ if it is permutation symmetric (i.e. $\phi(\bs) = \phi(P \bs)$ for any permutation matrix $P$) and satisfies Schur's condition
\begin{equation}\label{equ:schurscondition}
(s_i - s_j)\left(\frac{\partial \phi (\bs)}{\partial s_i}  - \frac{\partial \phi (\bs)}{\partial s_j}\right) \le 0, \forall i, j = 1, \ldots, N. 
\end{equation}
Furthermore, as $\phi (x)$ is assumed to be permutation symmetric, the above would be true if it holds for a single pair $(i, j)$ of specific values. See~\cite{marshall2010inequalities} for the proof.
\end{theorem}

Schur-concavity is well-known necessary condition for a function $\phi$ to be a good measure of diversity~\cite{marshall2010inequalities}. This class of functions maintain the preordering in reverse order. Thus a reasonable approach to maximizing the size of the inlier set is to minimize the diversity of shrinkage residuals (Theorem~\ref{th:mainth}), as measured by a Schur-concave function $\phi$.

\subsection{Diversity measures}

As motivated in the previous section, our task is to find a suitable choice of a Schur-concave function and minimize the corresponding objective function with the constraints in \eqref{equ:maxcon3}. We consider the Gaussian entropy measure and signomial diversity measure~\cite{kreutz1997general}.

\begin{definition}
The Gaussian entropy measure of diversity is
\begin{equation}\label{eq:dvrsty}
G_\gamma (\bs) = \sum_{i=1}^n \log (s_i + \gamma), ~~~~~~~ \gamma > 0. 
\end{equation}
\end{definition}
The Gaussian entropy~\cite{rao1999affine, kreutz1997general} has been studied for $\gamma = 0$. We introduce a small positive number $\gamma$ to ensure that the measure is bounded from below. This damping factor can also be observed as the  regularization of the optimization~\cite{chartrand2008iteratively}. In the following, we further prove that $G_\gamma$ satisfies Schur's condition~\eqref{equ:schurscondition}.
\begin{theorem}
$G_\gamma$ is strictly Schur-concave on the non-negative orthant $\mathbb{R}^n_+$. 
\end{theorem}
\begin{proof}
Let $P$ be a permutation matrix defined on a scalar vector $\bs \in  \mathbb{R}^n_+$. Then, $G_\gamma (P\bs) = \sum_{i=1}^n \log (p_i \bs + \gamma) = \sum_{i=1}^n \log (s_i + \gamma) = G_\gamma (\bs)$, where $p_i$ is the $i^{th}$ row of $P$. Hence $G_\gamma$ is permutation symmetric.  

For any pair of components $(i, j)$ of  a vector $\bs \in \mathbb{R}^n_+$ and for $\gamma > 0$, 
\begin{equation}
\begin{aligned}
& (s_i - s_j)\left(\frac{\partial G_\gamma (\bs)}{\partial s_i}  - \frac{\partial G_\gamma (\bs)}{\partial s_j}\right) \\  
 = & - \frac{(s_i - s_j)^2}{(s_i + \gamma)(s_j + \gamma)} \le 0, ~~ \text{since } s_i, s_j \ge 0 
 \end{aligned}
\end{equation}
The above would strictly follow the relation~\eqref{equ:schurscondition} if $s_i \neq s_j$. 
\end{proof} \qed


The above sigmoid measure have been utilized for the derivations of the proposed method IR-LP.

\subsection{M-estimators for robust statistics}\label{sec:irls}
In the context of other robust estimators such as M-estimators~\cite{huber2011robust}, iteratively reweighted least squares (IRLS) is well established as the optimizer. There are some recent methods \cite{aftab2015tpami,aftab15wacv} that utilizes IRLS for different geometric problems. However, there are fundamental and practical reasons to consider alternatives to IRLS for solving the MaxCon problem. M-estimators are well-studied in the field of robust statistics~\cite{huber2011robust}. The M-estimate is obtained as 
\begin{equation}\label{eq:mest} 
\argmin_{\btheta} \;\; \sum_{i=1}^n h \circ r(\btheta; \bx_i),
\end{equation}
where $h$ (called the M-estimator) is a symmetric, non-negative function with a unique minimum at zero. Standard M-estimators include Huber, Cauchy and Tukey robust costs; see Figure~\ref{fig:errorplot}. To solve~\eqref{eq:mest}, the classical IRLS method sequentially solves the weighted least squares problem 
\begin{equation}\label{eq:wls}
\begin{aligned}
\btheta^{(l+1)}  := & \argmin_{\btheta} \sum_{i=1}^n w^{(l)}_i r(\btheta; \bx_i)^2,\\
w^{(l)}_i  := &\frac{h^\prime(r(\btheta^{(l)};\bx_i))}{r(\btheta^{(l)};\bx_i)},
\end{aligned}
\end{equation}
where $h^\prime$ is the derivative of $h$. Aftab and Hartley~\cite{aftab15wacv} established the required properties of $h$ for IRLS to converge to a minimum of~\eqref{eq:mest}. {Note that there exists a closed form solution for each iteration of IRLS~\cite{gorodnitsky1997sparse} procedure under the linear residuals. Therefore, it is fast in linear case.} 

\begin{figure}[t]\centering
\includegraphics[width=0.6\columnwidth]{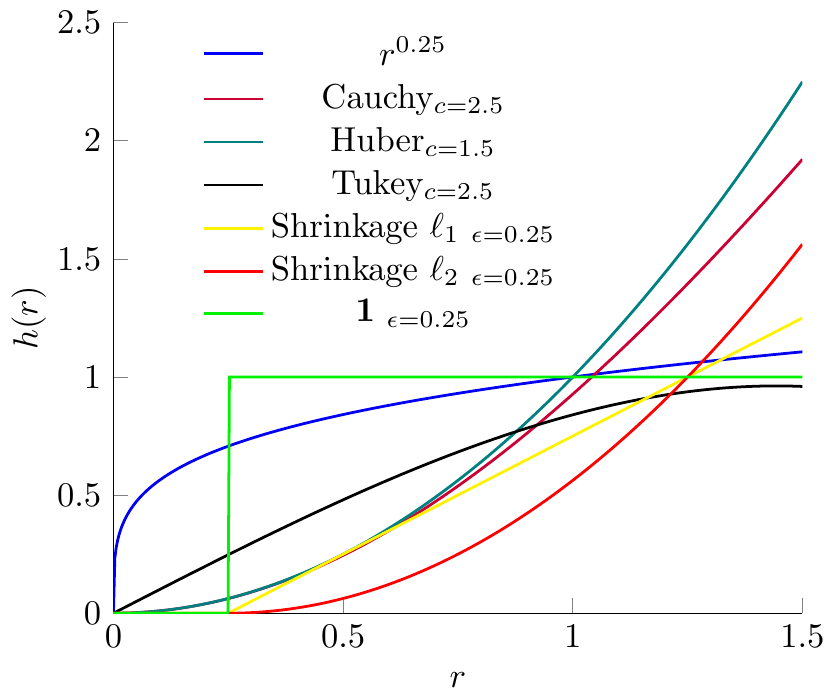}
\caption{Different M-estimators {$h(r)$} used for robust estimation including $\ell_1$-norm and $\ell_2$-norm of the shrinkage residuals. }
\label{fig:errorplot}
\end{figure}

Clearly,~\eqref{eq:mest}  reduces to MaxCon (P2, main manuscript) if $h$ is defined as
\begin{equation}\label{eq:hmaxcon}
h_\epsilon(r) = \mathbf{1}(r > \epsilon);
\end{equation}
see Figure~\ref{fig:errorplot} for a plot of $h_\epsilon$. However, $h_\epsilon$ does not satisfy the known properties of M-estimators for IRLS to guarantee convergence; see~\cite{daubechies2010iteratively,aftab15wacv} for details. Observe that we cannot even obtain useful weights in~\eqref{eq:wls} since $h^\prime_\epsilon$ is not defined everywhere, and where it is defined, $h^\prime_\epsilon(r) = 0$.

{Under the proposed choice of M-estimator} $h_\epsilon$ -- the formulation reduces to $\ell_0$ minimization of shrinkage residuals. We derive IRLS in terms of these shrinkage residuals in section \ref{sec:IRLS}. However, under linear residuals, unlike \eqref{eq:wls}, no closed form solution exists.

\section{Iterative Reweighted $\ell_2$ methods}\label{sec:IRLS}

Iterative reweighted least squares methods have been employed in \cite{chum03,lebeda2012fixing} for the refinement of the suboptimal \texttt{RANSAC} solutions. However, reweighted least squares is a well-known non-robust method \cite{aftab2015tpami,aftab15wacv} and we could certainly utilize similar idea and minimize the least squares of shrinkage residuals instead. \ie, 
\begin{equation} \label{eq:tlse}
\begin{aligned} 
\min_{\btheta,\; \bs} &
~~~ \sum_{i=1}^n s_i^2,  \\ 
~~\text{subject to} & 
~~~ r(\btheta; \bx_i) \leq \epsilon + s_i,\; s_i \ge 0, \\ 
\end{aligned} 
\end{equation}
which is also a convex optimization problem even under the linear residuals. However, again \eqref{eq:tlse} is a non-robust estimator according to the criterion for a robust estimator in \cite{aftab15wacv}; see Shrinkage $\ell_2$ in Figure \ref{fig:errorplot}. 

Let us robustify \eqref{eq:tlse} in the following manner 
\begin{equation}\label{eq:loganalysis2} 
\begin{aligned}
& ~~~~~~~~ \min_{\btheta,\; \bs}
& & \sum_{i=1}^n \log(s_i^2 + \gamma)  \\
& \text{subject to}
& & r(\btheta; \bx_i) \leq \epsilon + s_i,\; s_i \ge 0.\\
\end{aligned}
\end{equation}
which is also a smooth surrogate of MaxCon (P2, main manuscript). 

It can be verified easily that the KKT conditions, for $(\btheta^\ast,\bs^\ast)$ to be a solution of \eqref{eq:loganalysis2}, are the same as the weighted version of the least squares  \eqref{eq:tlse} of the shrinkage residuals with weight $w_i  := ({s_i^\ast}^2 + \gamma)^{-1}$. This leads to an IRLS estimation that minimizes \eqref{eq:loganalysis2} as follows
\begin{equation}
\left.
\begin{aligned}
 ~~~~~~(\btheta^{(l+1)}, \bs^{(l+1)}) := \argmin_{\btheta,\; \bs} \sum_{i=1}^n {w_i^{(l)}} s_i^2 \\
 \text{subject to~~~}  r(\btheta; \bx_i) \leq \epsilon + s_i, ~
 s_i \ge 0,~~~~~~\\
 ~~~~~~~~~~~~~~~w^{(l)}_i  := {(s_i^{(l)}}^2 + \gamma)^{-1}.~~~~~~~~
\end{aligned}
\right\rbrace \tag{S2} \label{eq:irlsmaxcon} 
\end{equation}
Each iteration of \eqref{eq:irlsmaxcon} is a quadratic program (QP) under linear or quasiconvex residuals. In the main draft of the paper, we call this method as IR-QP. Note that there is no closed form solution exists of each iteration and one needs to solve a convex quadratic program. 


{For $g_\gamma(x) = \log(x^2+\gamma)$, $g_\gamma^{\prime\prime}(x) = -\frac{2(x^2 - \gamma)}{(x^2+\gamma)^2}$. Thus, $g_\gamma(x)$ is convex in $(0, \sqrt{\gamma}]$ and concave in $[\sqrt{\gamma}, \infty)$. Therefore, it is unknown whether \eqref{eq:irlsmaxcon} converges to a minimum of~MaxCon (P2, main manuscript). The convergence analysis of the proposed method (S1) is addressed in the main manuscript.} From a practical standpoint, while the formulations of (S1, main manuscript) and \eqref{eq:irlsmaxcon} are quite similar, there are significant differences stated as follows:
\begin{itemize}
\setlength\itemsep{-0.25em}\setlength\parsep{-0.25em}
\item For each iteration, the objective of (S1, main manuscript) is linear while the objective of \eqref{eq:irlsmaxcon} is quadratic. Therefore, under linear residuals,~(S1, main manuscript) is an LP and~\eqref{eq:irlsmaxcon} is a QP. LPs are often faster than QPs.
\item It is well established that $\ell_1$ norm minimization tends to produce sparse results compared to $\ell_2$ norm. 

\end{itemize}
Experimentally, we observed that (S1, main manuscript) very frequently outperforms \eqref{eq:irlsmaxcon} given the same initializations.

\section{Additional Results}\label{sec:results}

To evaluate the proposed method IR-LP, in addition to the experiments in the main paper,  a number of experiments have been performed on synthetic and real datasets. We compared IR-LP against state-of-the-art approximate methods for MaxCon, with different initializations. \eg, 
\begin{itemize}
\item RANSAC + IR-LP: Proposed method IR-LP is initialized by the RANSAC solution. The runtime of RANSAC + IR-LP includes the runtime of IR-LP. The method is executed $100$ times and the average number of inliers found and the runtime are displayed in the table. 
\item $\ell_1$+IR-LP: Proposed method IR-LP is initialized by the  solution of iterative $\ell_1$~\cite{olsson10}. The runtime of $\ell_1$+IR-LP includes the runtime of IR-LP. 
\item $\ell_\infty$+IR-LP: Proposed method IR-LP is initialized by the  solution of iterative $\ell_\infty$~\cite{sim06}. The runtime of $\ell_\infty$+IR-LP includes the runtime of IR-LP. 
\item We also apply a locally optimized method LO-IR-LP, where we apply IRL1 for every successful RANSAC iterations. \ie, we apply IR-LP to refine the best solution found so far in RANSAC iterations. The maximum number of iterations were chosen to be $5$ for the inner loop. 
\end{itemize}

From the tables, it is very clear that proposed  IR-LP refines the outputs of the other methods with great extend. Overall LO-IR-LP produces better solution than RANSAC, however, we observe that better solution could be obtain by the refinement of the original RANSAC. 

Proposed method produces good results with almost any choice of initialization for homography estimation. It also produces much better solution for fundamental matrix estimation under a descent initialization. $\ell_1$+IR-LP does not work very well with poor initializations (iterative $\ell_1$~\cite{olsson10}). However, $\ell_\infty$+IR-LP works very well under relatively beter  initializations (iterative $\ell_\infty$~\cite{sim06}).


\newcolumntype{g}{>{\columncolor{gray}}r} 
\begin{sidewaystable}\setlength{\tabcolsep}{3pt}\renewcommand{\arraystretch}{1.3}\setlength{\tabcolsep}{3pt} 

\begin{center}
\small
  \begin{tabular}{p{1.85cm}rr|rr|rr|rr|rr|rr|rr|rr|rr} 
  \hline \hline
  \multicolumn{3}{c|}{}  & \multicolumn{2}{c}{}  & \multicolumn{2}{|c}{RANSAC} & \multicolumn{2}{|c}{} & \multicolumn{2}{|c}{$\ell_1$~\cite{olsson10} } & \multicolumn{2}{|c}{ } & \multicolumn{2}{|c}{$L_\infty$ ~\cite{sim06}} & \multicolumn{2}{|c}{LO-} & \multicolumn{2}{|c}{} \\ 
    \multicolumn{3}{c|}{Methods}  & \multicolumn{2}{c}{RANSAC}  & \multicolumn{2}{|c}{+ IR-LP} & \multicolumn{2}{|c}{$\ell_1$~\cite{olsson10} } & \multicolumn{2}{|c}{ + IR-LP } & \multicolumn{2}{|c}{$L_\infty$~ \cite{sim06}} & \multicolumn{2}{|c}{+ IR-LP} & \multicolumn{2}{|c}{RANSAC~ \cite{lebeda2012fixing}} & \multicolumn{2}{|c}{LO-IR-LP} \\ \hline 
     Datasets    			&   $n$  	& $|\cI^*|$   & $|\cI|$ ~~   & t (s)   & $|\cI|$           & t (s) & $|\cI|$         & t (s) & $|\cI|$ ~~          & t (s) & $|\cI|$  	 & t (s)     & $|\cI|$  	& t (s)     & $|\cI|$ ~~ 		& t (s)   & $|\cI|$~~ & t (s)\\ \hline 
     Val. Church & 108  	& 67    		& 64.48      & 0.046       	& 67.00 	& 0.072 & 26     		&     0.014   	 & 61	& 0.065 & 37       		& 0.021	  & 48 & 0.055 & 66.78		 & 0.066	& 66.33 & 0.247 \\ \hline 
     Uni. Library  & 665		& 552         & 542.31  	&  0.087             &  554.00   & 	0.357	& 435  		    & 	0.034 	&	554	& 0.147 & 251		    & 1.302	    & 554 & 1.507 &546.82		&0.274	& 554.00 & 0.387 \\ \hline 
     Keb. College  		& 399	& 311 		& 305.61 	& 0.067    	   	& 310.00	& 	0.126	& 102 			& 0.025 & 310	& 0.087  & 145 			& 0.145 	   & 310  & 0.218 &308.03		&0.137	& 310.12	&  0.361 \\ \hline 

Road Sign  		& 31 	& 29 		& 27.65 	& 0.006     	   	&  29.00	& 	0.012	& 14 			& 0.003 & 29	& 0.012  & 22 			& 0.002 	   & 29  & 0.012 & 28.61		& 0.011	 &  29.00	&  0.009 \\ \hline 

House  		& 492	& 355 		& 351.00 	& 0.031 	   	& 355.00 & 	0.250 	&  261	& 0.044 & 352 & 0.240  & 194	& 0.214  & 	351	& 0.292	 & 349.17  & 0.122 & 351.60	&  0.786 \\ \hline

     Cathedral  		& 544	& 481 		& 478.95  	& 0.043    	& 480.10	 & 	0.228	& 445	& 0.034 & 481 & 0.259 &  289		& 0.277      & 480		   & 0.373 & 473.41		& 0.193	& 479.25	& 0.395 \\

   \hline
\end{tabular}

\caption{Results for linearized homography estimation on Oxford VGG Datasets. $n$: number of point correspondences, $|\cI|$: is the consensus set size, $|\cI^*|$: optimal consensus set size, t(s): runtime in seconds }

~\\ 

 \begin{tabular}{p{2.5cm}r|rr|rr|rr|rr|rr|rr} 
  \hline \hline
  \multicolumn{2}{c|}{}  & \multicolumn{2}{c}{}  & \multicolumn{2}{|c}{RANSAC} & \multicolumn{2}{|c}{} & \multicolumn{2}{|c}{$\ell_1$ ~~\cite{olsson10} } & \multicolumn{2}{|c}{ } & \multicolumn{2}{|c}{$L_\infty$ ~\cite{sim06}} \\ 
    \multicolumn{2}{c|}{Methods}  & \multicolumn{2}{c}{RANSAC}  & \multicolumn{2}{|c}{+ IR-LP} & \multicolumn{2}{|c}{$\ell_1$ ~\cite{olsson10} } & \multicolumn{2}{|c}{ + IR-LP } & \multicolumn{2}{|c}{$L_\infty$ ~\cite{sim06}} & \multicolumn{2}{|c}{+ IR-LP} \\ \hline 
     Datasets    			&   n  	 & $|\cI|$    & t (s)   & $|\cI|$    & t (s)   & $|\cI|$     &  t (s)   & $|\cI|$  	& t (s)     & $|\cI|$  	& t (s) & $|\cI|$ & t (s) \\ \hline 
     Val. Church & 108  		& 77.20      & 0.084      &  85.00	& 0.480 & 26     		&  0.186    & 83 &  0.874  & 60       	& 1.221   &  85  & 4.924  \\ \hline

     Uni. Library  & 665		& 600.80  	&  0.113     	 &  611.35   &  	57.295	& 464  		    & 	2.744     	& 607  &  25.717 & 338		    & 8.421   & 611  &  25.344  \\ \hline 
     
     Keb. College  		& 399	& 307.35 			& 0.052    	 	& 307.85	& 10.646	& 92 & 0.639   	 &	308   & 6.420 & 177 		& 	2.126  &  309  & 7.816 \\ \hline 

      Road Sign  		& 34 	& 28.35 			&     0.010 	 	& 30.00 	& 1.425	& 2   & 0.558  	 &	2   & 2.052 & 23 		& 	0.242  &  30  & 1.636 \\ \hline 
      
       House  		& 492	& 353.05 			& 0.056    	 	& 354.95 	& 24.943 	& 273   & 1.922   	 &	355   & 24.842 & 277 		& 	1.278  &  353  & 24.855 \\ \hline 

        Cathedral  		& 544	& 474.00 	& 0.054    & 481.00 		& 	11.472	 	& 461	& 1.235	& 481   & 7.636  	 &	438   & 1.218   &  481  & 9.033 \\  
   \hline

\end{tabular}

\end{center}

\caption{Results for homography estimation with pseudo-convex residuals on Oxford VGG Datasets.}

\label{tab:linerHomography}  
\end{sidewaystable}

\newcolumntype{g}{>{\columncolor{gray}}r} 
\begin{sidewaystable}\setlength{\tabcolsep}{3pt}\renewcommand{\arraystretch}{1.3}\setlength{\tabcolsep}{3pt} 

\small

\begin{center}
  \begin{tabular}{p{2.5cm}rr|rr|rr|rr|rr|rr|rr|rr|rr} 
  \hline \hline
  \multicolumn{3}{c|}{}  & \multicolumn{2}{c}{}  & \multicolumn{2}{|c}{RANSAC} & \multicolumn{2}{|c}{} & \multicolumn{2}{|c}{$\ell_1$ ~\cite{olsson10} } & \multicolumn{2}{|c}{ } & \multicolumn{2}{|c}{$L_\infty$ ~\cite{sim06}} & \multicolumn{2}{|c}{LO-} & \multicolumn{2}{|c}{} \\ 
    \multicolumn{3}{c|}{Methods}  & \multicolumn{2}{c}{RANSAC}  & \multicolumn{2}{|c}{+ IR-LP} & \multicolumn{2}{|c}{$\ell_1$ ~\cite{olsson10} } & \multicolumn{2}{|c}{ + IR-LP } & \multicolumn{2}{|c}{$L_\infty$ ~\cite{sim06}} & \multicolumn{2}{|c}{+ IR-LP} & \multicolumn{2}{|c}{RANSAC ~\cite{lebeda2012fixing}} & \multicolumn{2}{|c}{LO-IR-LP} \\ \hline 
     Datasets    				&   n  	& $|\cI^*|$   & $|\cI|$    		& t (s)  & $|\cI|$    		& t (s)   	& $|\cI|$    		& t (s)   	& $|\cI|$           & t (s)  	& $|\cI|$  	& t (s)     & $|\cI|$  	& t (s)     & $|\cI|$  		& t (s) & $|\cI|$ & t (s) \\ \hline 
     Val. Church 			& 108  	 & 88 		& 76.64 			& 0.034 & 84.29	&	0.048  & 17			&	0.019   	 &  24 & 0.089 & 65   		&  0.0217 	& 86	& 0.042 & 80.59   		&  0.14 		 &  83.44 & 0.065 \\ \hline 
     Wad. College  			& 1051	& 365         & 287.84  	& 0.188     	   	&  343.40  &  0.341	 & 129  	    & 	0.025  	 		& 145	& 0.147 & 213		    & 0.0920     &  344    & 0.386 & 317.04		    &    0.22 		& 334.28 & 	0.683 \\ \hline 
     M. College I  			& 577		& 234 		& 212.6 		& 0.093      	   	&  232.00 & 0.142 & 79 		&  0.007		  	& 107 	& 0.132	   	& 58  		& 0.0325 	& 55 	& 0.142  &  222.84		&    0.167		  & 	230.6 & 0.361  	\\ \hline 
     
     M. College III  			& 313		& 214 		& 153.96 		& 0.082      	   	&  201.92 & 0.155 & 44 		&  0.003		  	& 55	& 0.053	   	& 174  		& 0.0175	& 210	& 0.248 &  184.9		&    0.083		  & 	199.92 & 0.117  	\\ \hline 
          
     Corridor  				& 124		& 72 			& 55.38 		& 0.153    	&  61.53	&  0.166
	 & 13 		&  0.006	 		& 18	&  0.046   & 55  		& 0.0094 	& 62 & 0.169 & 	59.46		&   0.128	 & 60.75	&  0.112\\ \hline 
	 
	 Dinosaur  		& 156		& 94 			& 68.60 		& 0.085    	&  85.00 	&  0.107	 &  	23	&  0.018 		& 33 &  0.043   & 70  		& 0.0129 	& 92 & 	0.037 & 	77.71		&   	0.072	 & 82.30 	&  0.093\\  
   \hline

\end{tabular}
\end{center}
\label{tab:linerFundamental}  
\caption{Results for linearized fundamental matrix estimation on Oxford VGG Datasets. $n$: number of point correspondences, $|\cI|$: is the consensus set size, $|\cI^*|$: optimal consensus set size, t(s): runtime in seconds} 
\end{sidewaystable} 

\bibliographystyle{splncs}

\bibliography{irl1}

\end{document}